\documentclass{article} %
\usepackage{arxiv_v1,times}

\usepackage{pgfplots}
\pgfplotsset{compat=1.18}
\usepgfplotslibrary{groupplots}

\usepackage{tikz}
\usepackage{comment}
\usetikzlibrary{positioning,arrows.meta,calc}

\definecolor{crghighfill}{HTML}{EEF7E8}
\definecolor{crghighborder}{HTML}{4B7F27}
\definecolor{crgmidfill}{HTML}{FFF5D6}
\definecolor{crgmidborder}{HTML}{A66A00}
\definecolor{crglowfill}{HTML}{FBE7E5}
\definecolor{crglowborder}{HTML}{A33A32}

\usepackage{fvextra}
\DefineVerbatimEnvironment{Prompt}{Verbatim}{
  breaklines=true,
  breakanywhere=true,
  fontsize=\scriptsize,
  frame=single,
  framesep=2mm
}

\usepackage{amsmath,amsfonts,bm}

\def\eqref#1{equation~\ref{#1}}

\def\1{\bm{1}}

\DeclareMathAlphabet{\mathsfit}{\encodingdefault}{\sfdefault}{m}{sl}
\SetMathAlphabet{\mathsfit}{bold}{\encodingdefault}{\sfdefault}{bx}{n}

\usepackage{url}

\usepackage{graphicx}
\usepackage{booktabs}
\usepackage{multirow}
\usepackage{colortbl}

\usepackage{comment}
\usepackage[colorlinks=true, citecolor=blue!50!black, linkcolor=blue!50!black]{hyperref}

\usepackage[normalem]{ulem}
\usepackage{xcolor}

\newif\ifshowremovals
\showremovalstrue

\newif\iffinalshort
\ifdefined\finalshort
  \finalshorttrue
\else
  \finalshortfalse
\fi
\finalshorttrue

\DeclareRobustCommand{\candidatesummaryremoval}[1]{%
  \iffinalshort\else
    \sout{#1}%
  \fi
}

\DeclareRobustCommand{\candidatesummaryadd}[1]{%
  \iffinalshort
    #1%
  \else
    \uline{#1}%
  \fi
}

\DeclareRobustCommand{\summaryremoval}[1]{%
}

\newcommand{\crg}{\mathcal{G}}
\newcommand{\claimset}{\mathcal{V}}
\newcommand{\childfn}{\mathrm{ch}}
\newcommand{\leafset}{\mathcal{L}}
\newcommand{\evidenceset}{\mathcal{X}}
\newcommand{\relfn}{\mathcal{R}}

\newcommand{\leafconf}{\psi}

\usepackage{amsthm}
\newtheorem{proposition}{Proposition}
\newtheorem{lemma}{Lemma}
\newtheorem{corollary}{Corollary}

\theoremstyle{remark}
\newtheorem{remark}{Remark}

\definecolor{gold}{rgb}{0.6, 0.4, 0}
\definecolor{darkred}{rgb}{0.5, 0.0, 0.0}
\definecolor{darkgreen}{rgb}{0.0, 0.5, 0.0}
\newcommand{\colordarkgreen}[1]{\textcolor{darkgreen}{#1}}
\newcommand{\colordarkgold}[1]{\textcolor{gold}{#1}}
\newcommand{\colordarkred}[1]{\textcolor{darkred}{#1}}

\newcommand{\brendan}[1]{\textcolor{violet}{[Brendan]: #1}}
\newcommand{\pouya}[1]{\textcolor{orange}{[Pouya]: #1}}
\newcommand{\estevam}[1]{\textcolor{team}{[Estevam]: #1}}
\newcommand{\farima}[1]{\textcolor{blue}{[Farima]: #1}}
\newcommand{\jf}[1]{\textcolor{gray}{[JF]: #1}}
\newcommand{\TODO}[1]{\textcolor{red}{[TODO]: #1}}
\newcommand{\claude}[1]{\textcolor{brown}{[Claude]: #1}}

\newcommand{\bdk}[1]{\textcolor{purple}{\bf\small [#1 --Brendan]}} 
\newif\ifshowauthorcomments
\showauthorcommentsfalse%

\ifshowauthorcomments\else
  \renewcommand{\bdk}[1]{}
  \renewcommand{\brendan}[1]{}
  \renewcommand{\pouya}[1]{}
  \renewcommand{\estevam}[1]{}
  \renewcommand{\farima}[1]{}
  \renewcommand{\jf}[1]{}
  \renewcommand{\TODO}[1]{}
  \renewcommand{\claude}[1]{}
\fi

\title{
Confidence Reasoning Graphs: Structured Confidence Estimation for LLM Agents}

\author{Brendan King\textsuperscript{1}\thanks{Work done while interning at Megagon Labs.},
Farima Fatahi Bayat\textsuperscript{2},  Jean-Flavien Bussotti\textsuperscript{2}, 
Pouya Pezeshkpour\textsuperscript{2}, \\ \textbf{Estevam Hruschka}\textsuperscript{2} \\
\normalfont
\textsuperscript{1}University of California, Santa Cruz, \textsuperscript{2}Megagon Labs\\
\texttt{bking2@ucsc.edu}, \texttt{\{farima,jflavien,pouya,estevam\}@megagon.ai}
}

\newcommand{\methodnamefull}[0]{Confidence Reasoning Graphs (CRGs)}
\newcommand{\methodnameshort}{CRG}

\definecolor{plotblue}{HTML}{0173B2}
\definecolor{plotorange}{HTML}{DE8F05}
\definecolor{plotgreen}{HTML}{029E73}
\newcommand{\plotours}[1]{\textcolor{plotblue}{#1}}
\newcommand{\plotverbalizer}[1]{\textcolor{plotorange}{#1}}
\newcommand{\plotsurrogate}[1]{\textcolor{plotgreen}{#1}}
\newcommand{\agenttraj}{\tau}
\newcommand{\trajspace}{\mathcal{T}}
\newcommand{\agentpolicy}{\pi}
\newcommand{\actionspace}{\mathcal{A}}
\newcommand{\obsspace}{\mathcal{O}}
\newcommand{\agentsuccess}{y}
\newcommand{\agentconf}{c}
\newcommand{\confestimator}{F_\theta}
\newcommand{\task}{q}          %
\newcommand{\env}{\Sigma}      %
\newcommand{\solution}{A}      %

\usepackage{amsthm}

\iclrfinalcopy %

\begin{document}

\maketitle

\begin{abstract}
When using an LLM agent in a consequential domain, making an informed decision about whether to trust its output or intervene requires calibrated confidence in the agent’s success.
Confidence estimation for agents is difficult because evidence about success is distributed across heterogeneous, interdependent steps of an agent's trajectory.
Practical agentic deployments introduce further challenges: frontier LLMs often provide limited access to internal signals, agent roll-outs are costly, and training data may be unavailable or quickly become outdated.
To address these challenges, we introduce \textbf{\methodnamefull}, an inference-time
framework that estimates the probability an agent accomplished its task from a
single trajectory, without privileged model access or training data.
Rather than compressing an execution into a single holistic judgment, a \methodnameshort{} begins with the claim that the agent accomplished its task, decomposes it into contextualized sub-claims grounded in trajectory evidence, estimates confidence for each terminal claim, and finally aggregates these into an overall confidence estimate.
Across three agentic benchmarks, three backbone models, and three agent frameworks, \methodnameshort{}s yield better-calibrated confidence and stronger risk-aware decision making than verbalized, sampling-based, and white-box surrogate baselines. 
We further find that calibration error alone can be misleading: a white-box surrogate baseline appears well calibrated while providing near-chance discrimination. 
Ablations attribute \methodnameshort{}’s improvements to claim-level confidence estimation and aggregation rather than graph construction alone.
Finally, a \methodnameshort{} exposes the claims and trajectory evidence underlying each confidence estimate, enabling it to be audited at decision time.\footnote{Code available at \href{https://github.com/megagonlabs/crg_ce}{https://github.com/megagonlabs/crg\_ce}}

\end{abstract}

\section{Introduction}
\label{sec:introduction}

Increasingly, consequential real-world tasks are being delegated to LLM agents: autonomous systems that combine LLM reasoning with tools, memory, and skills to solve complex, multi-step tasks through interaction with an environment.
Such agents are often used in domains where their decisions have real-world consequences, and users must correspondingly make inference-time decisions about whether to trust their output.
An agent's average success rate offers little guidance for this decision: it does not reveal whether a particular output is correct, and checking that output can be difficult because it results from a chain of decisions, tool calls, and observations in which one mistake can affect everything that follows.
What is needed is calibrated confidence for agents.

While confidence estimation for LLMs is well-studied \citep{geng_survey_2024}, the agentic setting poses new challenges.
Confidence in a single LLM output can often be estimated from sequence log-probabilities or from agreement across repeated samples \citep{wang2023selfconsistency}.
An agent trajectory, in contrast, is a heterogeneous trace that interleaves LLM-generated text, tool calls, skill, memory use, and observations from the environment.
Estimating confidence in task success therefore requires reasoning jointly over a long trajectory with different sources of uncertainty:
from the model, retrieval, tool failures, or environmental stochasticity \citep{liu-etal-2024-uncertainty, zhao-etal-2025-uncertainty}.

Real-world agent deployments further impose three constraints that make many existing confidence estimators for agents impractical.
First, many agents use black-box LLMs that do not expose internal states, token probabilities, or even LLM reasoning text, making estimators that rely on these privileged signals unusable.
Second, even a single agent run may be costly, making estimators requiring multiple roll-outs highly impractical.
Finally, training a separate estimator requires labeled trajectories, which are costly to collect and can become stale as agents, models, tools, and environments evolve.
Estimators reliant on token probabilities therefore cannot be applied to agents that use black-box LLMs \citep{zhao-etal-2025-uncertainty, zhang2026agentic, bouchard_beyond_2026}, and estimators based on agreement require an impractical number of runs for many real-world tasks \citep{duan_uprop_2025}.
While verbalized confidence estimators satisfy these constraints, the resulting estimates are often substantially overconfident \citep{kaddour_agentic_2026}.

To overcome these limitations, we introduce \textbf{\methodnamefull}, a confidence estimation method for agents that evaluates a single trajectory without requiring privileged model access, repeated roll-outs, or labeled training data. 
A \methodnameshort{} recursively decomposes the claim that \textit{the agent successfully accomplished its task} into a set of contextualized sub-claims that define success for the task, then estimates confidence in each leaf sub-claim grounded in evidence from the trajectory, before finally aggregating the resulting leaf confidences into a root confidence.
Together, these steps yield an overall confidence estimate and a structured argument supporting it.

We evaluate \methodnameshort{}s on trajectories produced by three strong agent LLMs across three challenging agent benchmarks: SWE-Bench Verified \citep{chowdhury_introducing_2024}, EnterpriseOps-Gym \citep{malay_enterpriseops-gym_2026}, and SkillsBench \citep{li_skillsbench_2026}.
\methodnameshort{}s achieve the lowest adaptive ECE and Brier score and the highest decision utility on every benchmark, outperforming both verbalized and surrogate-probability baselines.
Ablations attribute these gains to leaf-level confidence estimation and aggregation rather than graph construction alone. Additional experiments show that the gains transfer across agent frameworks and that CRGs are the best-calibrated method on the cost–calibration Pareto frontier.

Our contributions are as follows:
\begin{itemize}

\item We introduce \textbf{\methodnameshort{}s}, a structured, inspectable framework for training-free confidence estimation from a single black-box agent trajectory.
\item We evaluate \textbf{\methodnameshort{}s} across diverse agentic benchmarks, agent frameworks, and backbone models, and show that they provide substantially better calibrated confidence than strong verbalized and surrogate-probability baselines.
\item We validate each design component through ablations and show that \methodnameshort{}s lie on the cost--calibration Pareto frontier: the best-performing instantiation costs only \$0.04--\$0.07 per estimate, a fraction of the cost of generating a trajectory with a strong LLM.
\end{itemize}

\section{Related Work}
\label{sec:related-work}

\paragraph{Confidence Estimation for LLMs and Agents}
Prior work estimates confidence in LLM outputs using either ``white-box'' signals such as token probabilities and internal representations \citep{vashurin_benchmarking_2025, malinin2021uncertainty}, ``black-box'' signals based on agreement across sampled outputs \citep{wang2023selfconsistency, kuhn2023semantic}, or LLM verbalized confidence \citep{kadavath2022languagemodelsmostlyknow, tian-etal-2023-just, xiong2024can}.
Extending these approaches to \textit{LLM agents} requires estimating confidence over the heterogeneous, multi-step trajectories they produce.
A few works have approached this challenge, but require learning from or aggregating ``white-box'' signals \citep{bouchard_beyond_2026, zhang2026agentic, zhao-etal-2025-uncertainty, ma2026sequencestructurerelationaluncertainty}, using signals from repeated or branching executions \citep{mehta2026agentsdisagreethemselvesmeasuring, duan_uprop_2025}, or learning from labeled trajectories \citep{xuan-etal-2026-confidence}. 
These approaches therefore cannot be applied directly to the real-world setting described in \autoref{sec:introduction}, where agents use costly, black-box LLMs.
While verbalized confidence has been explored for agents, findings suggest they are highly over-confident \citep{kaddour_agentic_2026}.

\paragraph{Structured Arguments and Belief Propagation.}

Prior work in other areas assesses target propositions through structured arguments that connect claims to evidence.
Assurance cases, a closely related framework widely used in safety engineering, justify a critical system property through a hierarchical argument grounded in evidence \citep{Rushby2015TheIA}.
A few works have explored the construction or evaluation of assurance cases using LLMs \citep{ikhwantri_explainable_2025, ikhwantri_evaluating_2026}; however, these tasks differ substantially from estimating agent confidence, and none use the structure of the argument to quantify truth in the root claim.
Related forms of structure also appear in claim verification, where a target claim is expanded into graphs or trees of logically related sub-claims, including entailed and contradictory premises \citep{hou_probabilistic_2025}.
The Bayesian approach of \citet{hou_probabilistic_2025} models claim truth values as latent variables in a hidden Markov tree and infers target truth from observed LLM confidence scores. Its emission probabilities, $p(S_u \mid Z_u)$, are estimated from held-out truth-labeled statements, requiring labeled calibration data before inference. 
We instead target a setting requiring neither fitted parameters nor labeled calibration data.

\section{Preliminaries}
\label{sec:preliminaries}

\paragraph{LLM Agents} An LLM agent is an AI agent whose policy is governed by an LLM \citep{language-agent-tutorial}.
Given a task instruction $\task$, the agent interacts with an environment $\env$ over a sequence of steps.
At step $t$, the agent's policy $\agentpolicy$ conditions on the interaction history $h_t = (\task, a_1, o_1, \ldots, a_{t-1}, o_{t-1})$ to sample an action $a_t \sim \agentpolicy(\cdot \mid h_t)$, with $a_t \in \actionspace$. 
Applying $a_t$ to $\env$ returns an observation $o_t \in \obsspace$.
Execution halts after $T$ steps, either when the agent emits a terminal action or when a step budget is exhausted. 
The agent's execution is summarized by its \emph{trajectory}:
\begin{equation}
    \agenttraj = (\task,\, a_1, o_1,\, a_2, o_2,\, \ldots,\, a_T, o_T),
    \label{eq:trajectory}
\end{equation}
together with the \emph{solution} $\solution$ it produces, which may be an explicit final response within $a_T$, or a change the agent has made to $\env$, depending on the task.

\paragraph{Confidence Estimation} Let $\agentsuccess \in \{0, 1\}$ indicate whether the agent's solution $\solution$ successfully completes task $\task$, as determined by the environment's verifier. 
A \emph{confidence estimator} maps an observable trajectory $\agenttraj$ to a confidence score, $\confestimator : \trajspace \to [0,1]$, where $\agentconf = \confestimator(\agenttraj)$ estimates the probability of success, given the trajectory $P(\agentsuccess = 1 \mid \agenttraj)$.
We study this problem under three practical constraints which together characterize the real-world deployment setting introduced in \autoref{sec:introduction}:
\begin{enumerate}
    \item \textbf{Black-box:} The estimator reads only the textual content of $\agenttraj$, not hidden representations, token-level probabilities, or reasoning tokens.
    \item \textbf{Single-trajectory:} The estimator observes only one completed execution $\agenttraj$, and may not re-run the agent or interact with the environment $\env$.
    \item \textbf{Training-free:} The estimator requires no labeled trajectories, and hence no supervised fitting or post-hoc calibration.
\end{enumerate}
An estimator satisfying all three applies to any agent task, framework, and backbone model without instrumentation, and supports deployments with proprietary models in state-changing agents.

\paragraph{Evaluating a Confidence Estimator}
\label{sec:prelim-evaluation}

Confidence estimates serve different purposes, so we report complementary metrics.
Adaptive \textbf{Expected Calibration Error (ECE)} measures calibration by comparing mean confidence with empirical success rates in ten approximately equal-count bins.
\textbf{Area Under the Receiver-Operating characteristic Curve (AUROC)} measures whether successful trajectories receive higher confidence than unsuccessful ones, while the \textbf{Brier score} measures squared probabilistic error.
Because calibration and ranking alone do not capture the consequences of acting on an estimate, we also report the \textbf{Behavioral Alignment Score (BAS)} \citep{wu2026bas}, which averages accept-or-abstain utility across risk thresholds and strongly penalizes confident failures.
Formal definitions of ECE and Brier score appear in \autoref{app:metrics}; the BAS decision model appears in \autoref{app:bas}.

\begin{figure*}[t]
    \centering
    \includegraphics[width=\linewidth]{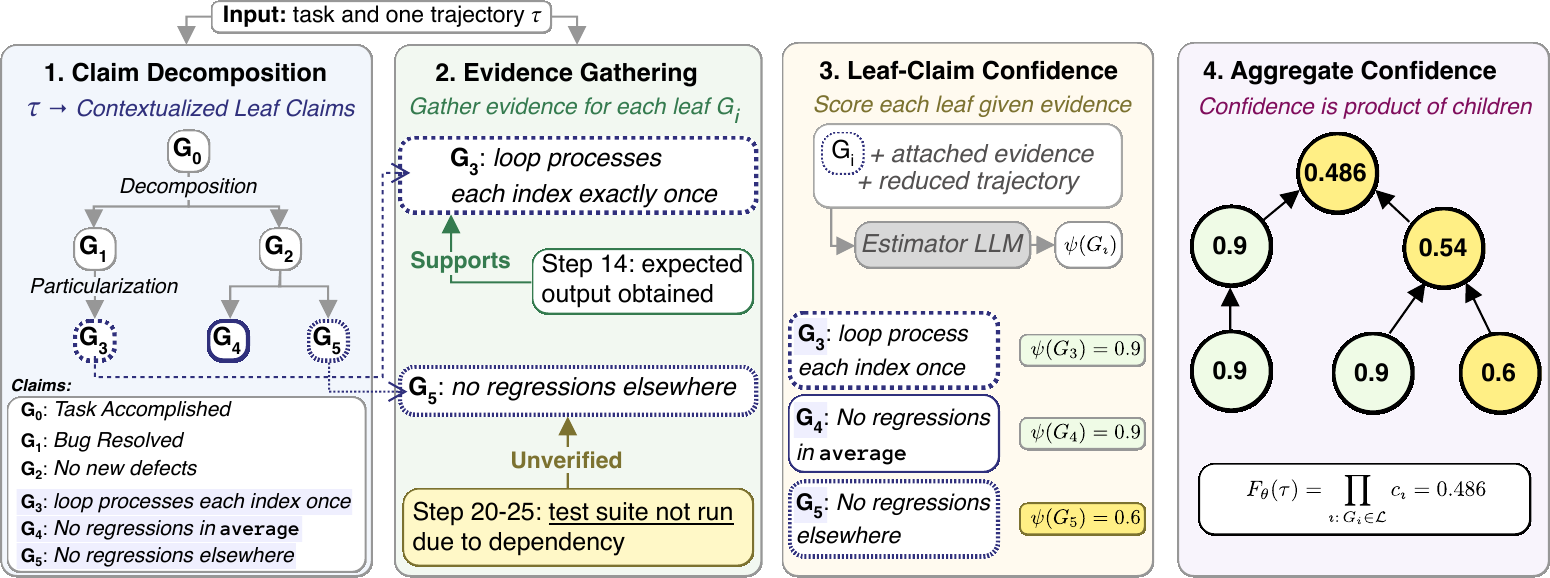}
    \caption{Overview of the \methodnameshort{} confidence estimation procedure (\autoref{sec:methods-graph-claim-decomp}-\autoref{sec:methods-confidence-aggregation}). Given a single task and trajectory $\agenttraj$, a constructor LLM (1) decomposes and particularizes the task-success claim and then (2) attaches cited evidence from $\agenttraj$ to the resulting leaf claims, using labels \colordarkgreen{Supports}, \colordarkgold{Unverified}, and \colordarkred{Undermines} (two types shown). In (3), an estimator LLM assigns a confidence $\psi(G_i)$ to each leaf given its evidence and the trajectory reduced according to the cited steps. In (4), confidence in each non-leaf claim is computed as the product of its children. \methodnameshort{}s support estimation and inspection: low-confidence paths can be traced to uncertain leaf claims and their associated evidence.}
    \label{fig:method-overview}
\end{figure*}

\section{Methodology}
\label{sec:methods}

Given a trajectory $\agenttraj$, we re-frame confidence estimation as evaluating confidence in the claim that the agent accomplished its task \textit{successfully.}
\methodnameshort{}
evaluates
this claim using a structured argument inspired by \textit{assurance cases}
in four steps, illustrated in \autoref{fig:method-overview}.
First, we use an LLM to recursively decompose it into simpler, contextualized claims which jointly define task success (\autoref{sec:methods-graph-claim-decomp}).
Second, an LLM grounds each resulting leaf claim in evidence extracted from the trajectory (\autoref{sec:methods-graph-evidence-gathering}).
We then estimate the confidence in the truth of each leaf claim using an LLM (\autoref{sec:methods-verbalizing}), and finally aggregate these estimates into confidence in the original claim (\autoref{sec:methods-confidence-aggregation}).
As in assurance cases, our motivation is that a high-level, multifaceted claim 
is assessed more reliably
when decomposed into concrete, evidence-grounded sub-claims. 
Formally, let $G_0$ denote the claim that \textit{the agent successfully completed
its task}. \methodnameshort{} estimates confidence in the probability that this claim is true given the trajectory:
\[
    \confestimator(\agenttraj)
    \;\approx\;
    P(G_0 \mid \agenttraj).
\]

\subsection{Claim Decomposition}
\label{sec:methods-graph-claim-decomp}

Starting from $\claimset = \{G_0\}$, an LLM iteratively expands a leaf %
claim by applying one of two refinement strategies, both inspired by argument strategies in the  Goal Structuring
Notation assurance case formalism \citep{kelly_weaver_gsn_2004}:

\begin{enumerate}
    \item \textbf{Decomposition} splits a claim $G \in \claimset$ into individually necessary and jointly sufficient sub-claims $\childfn(G)=\{G_1,\ldots,G_m\}$, such that $G \equiv \bigwedge_{j=1}^{m}G_j$ and the sub-claims are mutually independent given $\agenttraj$.
    \item \textbf{Particularization} restates a claim $G \in \claimset$ as one contextualized for the specific task, yielding one equivalent child $\childfn(G)=\{G'\}$ such that $G \equiv G'$.
\end{enumerate}

Decomposition splits a broad claim into constituents that can be checked separately; particularization restates a generic claim as one that can be evaluated in the context of the given task $\task$ and trajectory $\agenttraj$.
As an example, consider a coding agent tasked with fixing an off-by-one bug in a function \texttt{average}.\footnote{An off-by-one error is
a common class of bug in which a loop or index boundary is set one position too
high or too low --- here, a summation loop that skips the final element, so
that \texttt{average([1,2,3])} returns $1.5$ rather than $2.0$.}
The root claim $G_0$ might be decomposed into ``the reported bug has been resolved'' and ``no new defects were introduced,'' and the former particularized as ``the modified loop processes every intended index exactly once.''
Both strategies aim to preserve the same invariant: each parent claim is equivalent to the conjunction of its children:
\begin{equation}
    G \;\equiv\; \bigwedge_{G' \in \childfn(G)} G' .
    \label{eq:refinement-invariant}
\end{equation}
Applying this invariant recursively, the root claim $G_0$ is equivalent to the conjunction of the tree's leaves $\leafset$: 
\begin{equation}
    G_0 \;\equiv\; \bigwedge_{G \in \leafset} G .
    \label{eq:root-equiv}
\end{equation}
We formalize this equivalence and its implications for confidence aggregation in \autoref{app:detailed_proof}. 
The LLM may refine any current leaf or stop early, subject to a maximum
refinement depth $k$, where depth is the number of refinement edges on a
root-to-leaf path.
We set $k=5$ throughout; \autoref{app:depth} examines sensitivity to this choice.
\autoref{sec:analysis-graph-quality} evaluates how well constructed graphs
satisfy \autoref{eq:refinement-invariant} in practice.

\subsection{Evidence Gathering}
\label{sec:methods-graph-evidence-gathering}
To ground confidence estimation for the leaf claims, we enrich the claim tree induced by $\childfn$ with evidence from the trajectory, forming a graph.
In a single prompt conditioned on both $\agenttraj$ and this tree, an LLM produces a set of evidence items $\evidenceset$ and attaches each item $x \in \evidenceset$ to one or more leaf claims in $\leafset$.
Each evidence item $x$ describes a salient event from the trajectory for evaluating claims.
Each attachment cites a set of trajectory steps and receives one of three labels: $\texttt{Supported}$, $\texttt{Undermined}$, or $\texttt{Unverified}$.
We represent these labeled attachments as
\[
\relfn \subseteq
\evidenceset
\times \leafset
\times 2^{\{1,\ldots,T\}}
\times
\{\texttt{Supported},\texttt{Undermined},\texttt{Unverified}\},
\]
 where $(x,G,S,r)\in\relfn$ means that evidence item $x$ is attached to leaf claim $G$ citing trajectory steps $S$ with label $r$. 
Our relations follow work on evidence-based factuality evaluation \citep{fatahi-bayat-etal-2025-factbench, song-etal-2024-veriscore}, where each piece of evidence supports belief in claim, undermines belief in that claim, or indicates that a claim cannot be verified from the trajectory.
Together, $\claimset\cup\evidenceset$ form the graph's nodes, while $\childfn$ and $\relfn$ define its claim and evidence edges, respectively.

Returning to the running example, the claim ``the modified loop processes every
intended index exactly once'' may be $\texttt{Supported}$ by step 14, where the
agent runs the repaired function and obtains the expected output, and
$\texttt{Unverified}$ by steps 20--25, where the agent abandons efforts to run the project's test
suite after a dependency fails to install.
$\texttt{Unverified}$ evidence matters because a common failure mode of agents is not by acting
incorrectly, but stopping short of confirming that an action succeeded.
The prompt for claim decomposition and evidence gathering appears in \autoref{app:crg_construction}.

\subsection{Estimating Leaf Confidence}
\label{sec:methods-verbalizing}

Let $\leafconf : \leafset \to [0,1]$ assign each leaf claim $G \in \leafset$ a
confidence $c_G = \leafconf(G)$, interpreted as an estimate of
$P(G \mid \agenttraj)$.
We instantiate $\leafconf$ by verbalized confidence, prompting an estimator
LLM for the probability that $G$ holds.

\candidatesummaryadd{Each leaf is scored in a separate call using its attached evidence and a reduced trajectory. We retain the prefix through the last cited step, using $T$ as the cutoff when no step is cited. Uncited actions are replaced by framework-provided summaries when available and truncated otherwise. Summarizing them preserves context, while keeping the cited steps the most salient content in the prompt.
The reduction assumes that the latest relevant evidence was cited; it cannot recover evidence omitted during construction. Prompts appear in \autoref{app:prompts}.}

\subsection{Aggregating Confidence to the Root}
\label{sec:methods-confidence-aggregation}

Given leaf confidences, we define a confidence $f(G)$ for every claim in
$\claimset$ by recursion from the leaves:
\begin{equation}
    f(G) \;=\;
    \begin{cases}
        \leafconf(G), & G \in \leafset, \\[4pt]
        \displaystyle\bigoplus_{G' \in \childfn(G)} f(G'), & \text{otherwise,}
    \end{cases}
    \label{eq:recursion}
\end{equation}
where $\bigoplus$ is an aggregation operator over the confidence values of a claim's
children, and the estimator returns the confidence at the root,
$\confestimator(\agenttraj) = f(G_0)$.
This formulation lets us compare aggregation rules over identical graphs, which we do in Appendix~\ref{app:aggregation}.

Under \autoref{eq:root-equiv}, the root claim holds exactly when all leaf
claims hold, so its probability is their joint probability.
Computing this joint probability generally requires modeling dependencies
among the leaves.
For aggregation, we make the simplifying approximation that the children introduced at each decomposition are conditionally independent given $\tau$, yielding the approximation
$P(G_0\mid\agenttraj) \approx
\prod_{G\in\leafset}P(G\mid\agenttraj)$.
Replacing each marginal probability with its leaf estimate
$c_G \approx P(G\mid\agenttraj)$ gives:
\begin{equation}
    \confestimator(\agenttraj)
    \;=\;
    f(G_0)
    \;=\;
    \prod_{G\in\leafset} c_G .
    \label{eq:aggregation}
\end{equation}
Thus, our final confidence estimate is the product of the leaf confidences.
\autoref{app:detailed_proof} gives the general dependence-aware
factorization, marginal-only bounds, and derives the product rule under our local
conditional-independence approximation.

\section{Experimental Setup}
\label{sec:experiments}

\paragraph{Benchmarks and Trajectories.}
We evaluate confidence estimators on trajectories from three challenging benchmarks: SWE-Bench Verified~\citep{chowdhury_introducing_2024} for software engineering, EnterpriseOps-Gym~\citep{malay_enterpriseops-gym_2026} for tool-heavy enterprise workflows, and SkillsBench~\citep{li_skillsbench_2026} for long-horizon skill use. We focus on harder problems, where task success is less saturated and confidence estimation is more informative; \autoref{app:datasets} details the filtering criteria and splits. \autoref{tab:agent-success} summarizes the resulting data. For each problem, we collect one trajectory using the OpenHands agent framework~\citep{wang2025openhandssoftwareagentsdk}, %
with benchmark-specific iteration caps (\autoref{app:datasets}), with three agent LLMs: GPT-5.5~\citep{openai2026gpt55}, Gemini-3.5 Flash~\citep{google2026gemini35flash}, and the open-weight MiniMax-M3~\citep{lai2026minimaxsparseattention}. Each benchmark's verifier provides the success label $\agentsuccess$ used to evaluate confidence estimates. We evaluate transfer beyond OpenHands in \autoref{sec:results-other-frameworks}.

\begin{table}[t]
\centering
\caption{OpenHands success rates (\%) by agent LLM. $n$ is the number of problems.}
\label{tab:agent-success}
\small
\resizebox{\textwidth}{!}{%
\begin{tabular}{lcccccc}
\toprule
Benchmark & Num. Problems & Total Trajectories & GPT-5.5 & Gemini-3.5 Flash & MiniMax-M3 & Avg. \\
\midrule
SWE-Bench Verified & 306 & 900 & 73.9 & 74.3 & 71.7 & 73.3 \\
EnterpriseOps-Gym  & 325 & 975 & 41.8 & 44.3 & 31.4 & 39.2 \\
SkillsBench        & 81  & 219 & 50.7 & 50.7 & 45.2 & 48.9 \\
\bottomrule
\end{tabular}%
}
\end{table}

\paragraph{CRG Settings.}
We run \methodnameshort{} with two estimator LLMs, which read each completed trajectory and output a confidence: the open-weight Qwen-3.8 27B and the closed-source GPT-5.6 Sol \citep{openai2026gpt56}, both set to medium reasoning effort.
We use a maximum refinement depth of $k=5$ across
benchmarks without per-benchmark tuning; \autoref{app:depth} shows that no graph exceeds this depth in a development-set ablation.

\paragraph{Baselines.}
We compare CRG against three black-box baselines and one surrogate-probability baseline.
\emph{Basic Verbalizer} directly asks the estimator LLM for its confidence that the agent succeeded~\citep{zhang2026agentic}.
\emph{Reason-as-Graph} describes our claim-decomposition procedure in the estimator prompt, testing whether decomposition-oriented reasoning alone recovers the benefits of explicit graph construction and leaf aggregation.
Following self-consistency work~\citep{wang2023selfconsistency}, \emph{Verbal Consistency} samples ten true/false judgments of task success \textit{with reasoning} and uses the fraction judged true as confidence.\footnote{Due to inference cost, we run Verbal Consistency only with Qwen-3.8 27B.}
Finally, \emph{Surrogate LNSP} uses the length-normalized sequence probability of the agent's final action under a surrogate LLM~\citep{shrivastava_llamas_2023,zhang2026agentic,bouchard_beyond_2026}.
\autoref{app:surrogate-aggregations} provides comparisons to aggregations of surrogate token probabilities.
All black-box baselines use the same backbone LLM as the \methodnameshort{} constructor and estimator. \autoref{app:prompts} provides the complete prompts.

\paragraph{Metrics.}
We report adaptive ECE with 10 approximately equal-count bins, Brier score, AUROC, and BAS. Their definitions appear in \autoref{app:metrics} and \autoref{app:bas}; \autoref{app:ece-bin-sweep} reports equal-width ECE and sensitivity to bin count.

\section{Results and Analyses}
\label{sec:results}
We compare \methodnameshort{} against black-box and surrogate probability baselines in Section~\ref{sec:results-main}, test transfer to two other frameworks in Section~\ref{sec:results-other-frameworks}, and analyze the contribution of key components of the framework through ablations in Section~\ref{sec:analysis-component-ablation}.

\subsection{Main Results}
\label{sec:results-main}

\begin{table*}[t]
\centering
\caption{Confidence estimation results on OpenHands trajectories, averaged over agent LLMs. We report adaptive ECE (10 bins), Brier Score, AUROC, and BAS. Best per column in \textbf{bold}, second \underline{underlined}. $^\dagger$White-box: uses token probabilities from a surrogate LLM.}
\label{tab:main-results}
\setlength{\tabcolsep}{3pt}
\resizebox{\textwidth}{!}{%
\begin{tabular}{ll cccc cccc cccc}
\toprule
& & \multicolumn{4}{c}{SWE-Bench Verified} & \multicolumn{4}{c}{EnterpriseOps-Gym} & \multicolumn{4}{c}{SkillsBench} \\
\cmidrule(lr){3-6} \cmidrule(lr){7-10} \cmidrule(lr){11-14}
Estimator & LLM & ECE$\downarrow$ & Brier$\downarrow$ & AUROC$\uparrow$ & BAS$\uparrow$ & ECE$\downarrow$ & Brier$\downarrow$ & AUROC$\uparrow$ & BAS$\uparrow$ & ECE$\downarrow$ & Brier$\downarrow$ & AUROC$\uparrow$ & BAS$\uparrow$ \\
\midrule
\multirow{2}{*}{Basic Verbalizer} & GPT-5.6 Sol & 0.20 & 0.23 & 0.57 & $-$0.04 & 0.58 & 0.57 & 0.50 & $-$2.48 & 0.46 & 0.46 & 0.54 & $-$2.17 \\
 & Qwen-3.8 27B & 0.18 & 0.22 & \textbf{0.66} & 0.27 & 0.49 & 0.47 & 0.60 & $-$0.57 & 0.41 & 0.41 & \textbf{0.65} & $-$0.33 \\
\multirow{2}{*}{Reason-as-Graph} & GPT-5.6 Sol & 0.17 & 0.22 & 0.57 & 0.06 & 0.57 & 0.56 & 0.47 & $-$2.69 & 0.43 & 0.44 & 0.51 & $-$2.09 \\
 & Qwen-3.8 27B & 0.18 & 0.22 & \textbf{0.66} & 0.27 & 0.46 & 0.44 & \underline{0.62} & $-$0.50 & 0.38 & 0.38 & \underline{0.63} & $-$0.40 \\
Verbal Consistency & Qwen-3.8 27B & 0.24 & 0.25 & 0.54 & $-$7.58 & 0.45 & 0.45 & \underline{0.62} & $-$13.72 & 0.45 & 0.44 & 0.60 & $-$13.87 \\
Surrogate LNSP$^\dagger$ & Qwen-3.8 27B & \underline{0.10} & \underline{0.21} & 0.47 & \underline{0.36} & \underline{0.22} & \underline{0.28} & 0.61 & 0.04 & \underline{0.16} & \underline{0.28} & 0.46 & \underline{0.09} \\
\midrule
\rowcolor{blue!8} & GPT-5.6 Sol & 0.16 & 0.23 & \underline{0.62} & \underline{0.36} & 0.24 & 0.29 & \textbf{0.67} & \underline{0.06} & 0.22 & 0.29 & \textbf{0.65} & \underline{0.09} \\
\rowcolor{blue!8} \multirow{-2}{*}{\textbf{CRG (ours)}} & Qwen-3.8 27B & \textbf{0.09} & \textbf{0.20} & 0.57 & \textbf{0.37} & \textbf{0.13} & \textbf{0.24} & \underline{0.62} & \textbf{0.08} & \textbf{0.11} & \textbf{0.24} & \textbf{0.65} & \textbf{0.14} \\
\bottomrule
\end{tabular}
}
\end{table*}

\paragraph{Overall Performance.}
\autoref{tab:main-results} shows a consistent advantage for CRGs in calibration and decision quality across all three benchmarks.
\textbf{CRGs achieve the lowest adaptive ECE and Brier score and highest BAS on every benchmark}, with a CRG also attaining the best or second-best AUROC.
CRGs with Qwen-3.8 27B are best calibrated, achieving adaptive ECEs of 0.09, 0.13, and 0.11 on SWE-Bench Verified, EnterpriseOps-Gym, and SkillsBench, respectively, compared with 0.18--0.49 for direct verbalization and 0.10--0.22 for the surrogate-probability baseline.
\methodnameshort{}s with GPT-5.6 Sol substantially outperform their verbalized counterparts, indicating that improvements are not specific to a single \methodnameshort{} instantiation. 
We also observe divergence between calibration and discrimination: direct verbalizers reach AUROC 0.66 on SWE-Bench Verified despite substantial miscalibration, whereas Surrogate LNSP attains lower ECE but weak AUROC, performing near-chance on SWE-Bench Verified (0.47) and SkillsBench (0.46).
\textbf{\methodnameshort{} is the only evaluated approach that consistently combines strong calibration (ECE $\downarrow$), competitive discrimination (AUROC $\uparrow$), and positive expected utility (BAS $\uparrow$) across all three benchmarks.}
Bootstrap resampling supports these gains (\autoref{app:paired-bootstrap}).
We further find these improvements are inexpensive: \methodnameshort{} with Qwen 3.8 27B costs \$0.04--0.07 per trajectory and lies on the cost--calibration Pareto frontier, while verbalized baselines remain less calibrated even at higher reasoning effort and comparable cost (\autoref{app:analysis-cost}).

\paragraph{Evaluating Reliability.}
To investigate the source of calibration differences in \autoref{tab:main-results}, \autoref{fig:reliability} plots reliability curves and empirical confidence CDFs for the best-calibrated verbalizer, surrogate, and \methodnameshort{} instantiation.
The represented verbalizer (Reason-as-Graph) is systematically overconfident, assigning confidence scores over 90\% to the majority of trajectories on all benchmarks despite substantially lower success rates.
\autoref{fig:reliability} further reveals why calibration error alone is insufficient as a metric.
On SWE-Bench Verified, Surrogate LNSP's adaptive ECE is close to \methodnameshort{}'s (0.10 vs.\ 0.09), but its predictions are concentrated near the benchmark success rate (mean confidence 0.69 vs.\ 0.73 success rate).
This yields low calibration error without reliably distinguishing individual successes from failures, as reflected in its poor AUROC.
\textbf{\methodnameshort{}s instead distribute confidence estimates across a substantially wider range while remaining closer to the calibration diagonal on all benchmarks.}

\paragraph{Evaluating Decision Utility.}
To evaluate how confidence estimates inform decision-making under risk, \autoref{fig:utility-curves} reports selective utility across acceptance thresholds $t$, with BAS corresponding to the area under each curve.
A utility below zero means that for the risk-reward ratio defined by $t$, it would have been better to abstain altogether than to use that estimator for decision making.
At low thresholds (low-risk settings), all methods behave similarly. 
However, their behavior diverges as risk increases: 
Consistent with the overconfidence observed in \autoref{fig:reliability}, Reason-as-Graph incurs \textit{negative utility} on all benchmarks due to high-confidence failures, reaching $-3.6$ on EnterpriseOps-Gym. 
Surrogate LNSP avoids incurring costs of this magnitude, but achieves lower utility than \methodnameshort{}s at nearly every risk threshold.
\textbf{\methodnameshort{}s achieve the highest BAS on all three benchmarks, maintain positive utility for the widest span of risk thresholds, 
and are the only method that never falls substantially below total abstention.}
\methodnameshort{}s therefore support the best decision-making in both low-risk and high-risk environments.

\begin{figure*}[t]
\centering
\includegraphics[width=\linewidth]{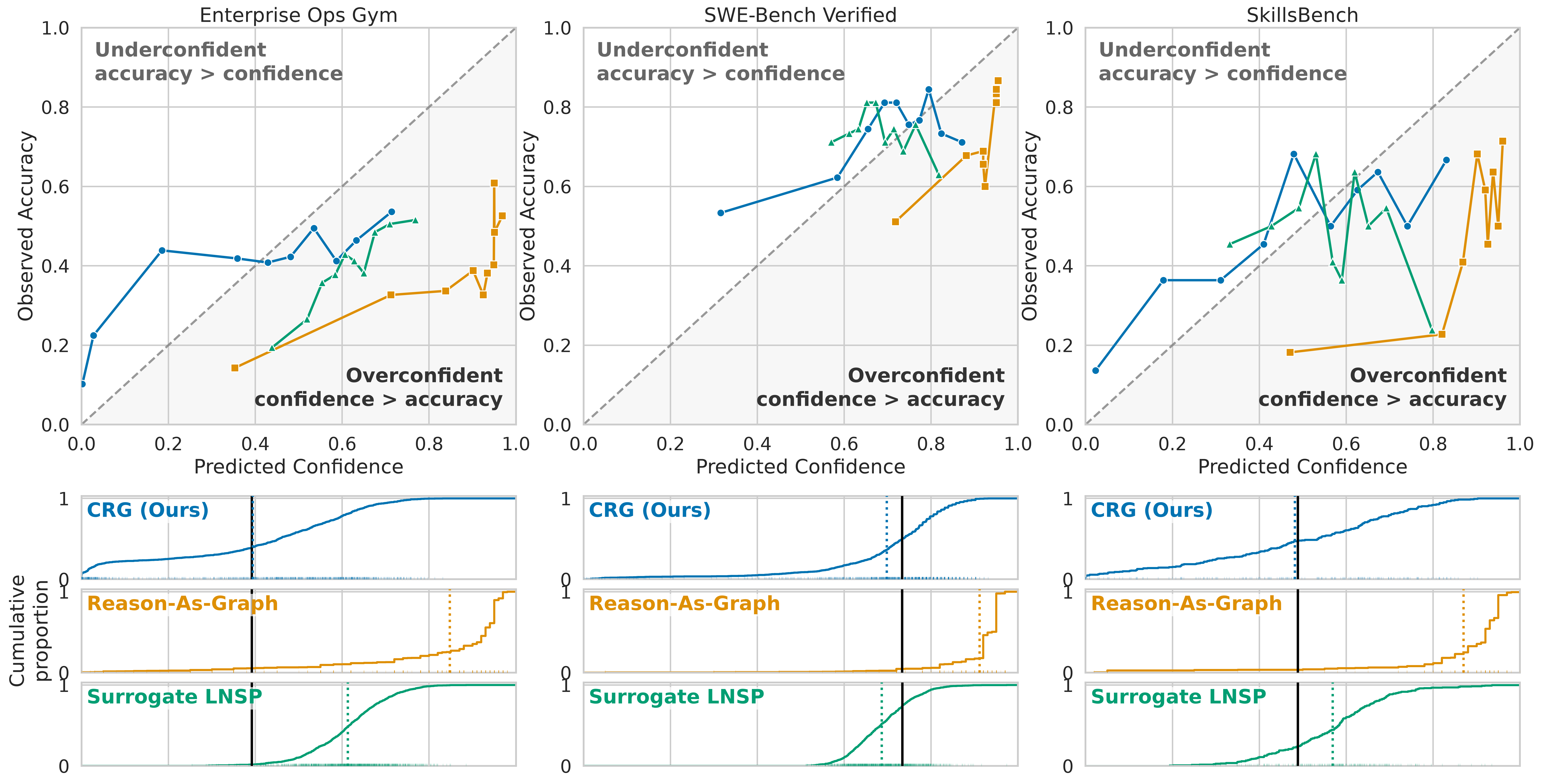}
\caption{\candidatesummaryremoval{Reliability diagrams (top) for the best method from each class on each benchmark. Perfectly calibration lies on the diagonal, with under-confidence lying above and over-confidence (shaded) below. Histograms (bottom) plot the \textit{distribution} of confidence estimates, with the dashed line indicating mean estimated confidence for that system, and the solid black line indicating the base success rate for that benchmark. For plot legibility, we use 8 evenly spaced bins.
We find \methodnameshort{}s (\plotours{Ours}) to be best calibrated for each benchmark, with the lowest ECE of any estimator (0.13, 0.09, and 0.11, left to right) and confidence estimates spread across the range rather than massed in a single bin.
Across benchmarks, the best verbalizer (\plotverbalizer{Reason-As-Graph}) is highly over-confident: placing most trajectories in its highest-confidence bin, well above their accuracy. 
The surrogate baseline (\plotsurrogate{Surrogate LNSP}) is reasonably well calibrated, but concentrates its estimates in a narrow band around the base rate, leaving it unable to separate successes from failures (AUROC 0.47 on SWE-Bench Verified).
} \candidatesummaryadd{Reliability diagrams (top) and empirical confidence CDFs (bottom) for \plotours{CRG (Ours)}, \plotverbalizer{Reason-as-Graph}, and \plotsurrogate{Surrogate LNSP}, using 10 equal-mass bins. Perfect calibration lies on the diagonal; under-confidence is above and overconfidence below (shaded). Dashed and solid vertical black lines indicate mean confidence and benchmark success rate, respectively. \methodnameshort{} is best calibrated on all three benchmarks, with estimates distributed across the confidence range. }
}
\label{fig:reliability}
\end{figure*}

\begin{figure*}[t]
\centering
\includegraphics[width=\linewidth]{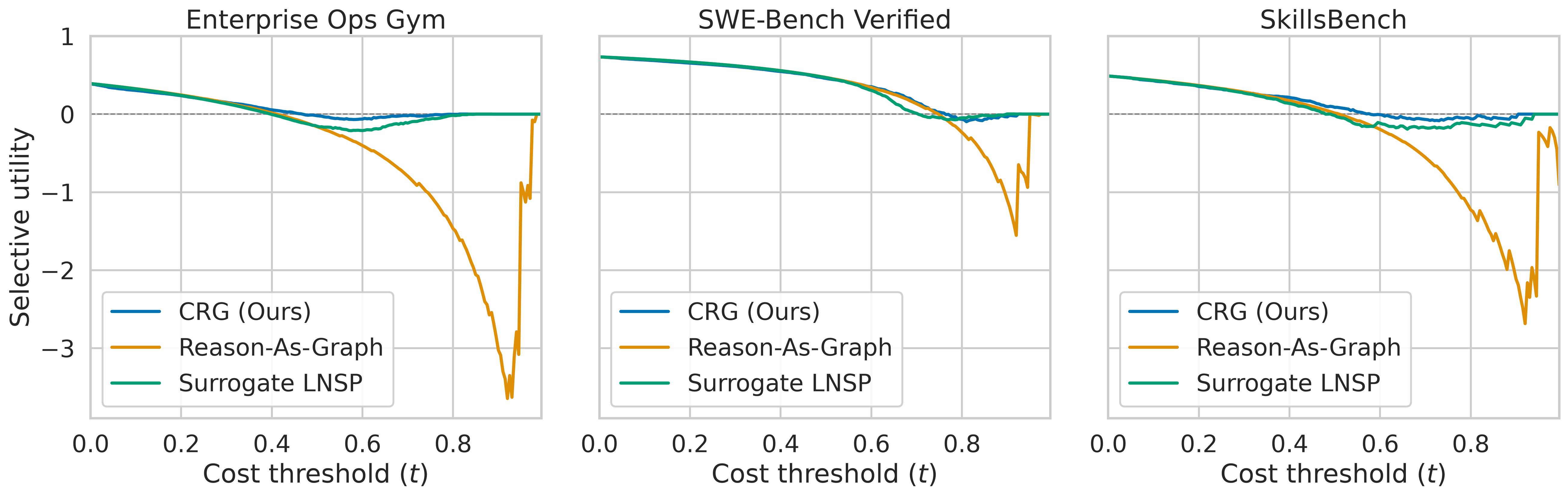}
\caption{Selective utility across cost thresholds $t$ for the best method from each class (all Qwen-3.8 27B). The dashed zero line corresponds to abstaining from all solutions, so values below it are worse than abstention. All estimators coincide at low $t$, where accepting all solutions is optimal. As $t$ increases, \plotverbalizer{Reason-As-Graph} incurs losses as large as $-3.6$ on EnterpriseOps-Gym, while \plotours{\methodnameshort{} (Ours)} remains above \plotsurrogate{Surrogate LNSP} through the mid range.}
\label{fig:utility-curves}
\end{figure*}

\subsection{Generalization to Other Agent Frameworks}
\label{sec:results-other-frameworks}

In this section, we evaluate whether \methodnameshort{} transfers to agent frameworks beyond OpenHands. We evaluate on 609 SWE-Bench Verified trajectories collected with two additional framework--model combinations: Codex with GPT-5.5 and Claude Code with Claude Opus 4.7. 
We compare against the black-box verbalized baselines; Surrogate LNSP is omitted because computing its token probabilities requires reconstructing the exact LLM context at each agent step.\footnote{This context may not be recoverable post hoc, for cloud-hosted or closed agent frameworks, making surrogate-based methods difficult to apply in such settings.}

\autoref{tab:other-frameworks} shows that the calibration and decision-utility gains transfer across agent frameworks %
\textbf{\methodnameshort{}}s reduce adaptive ECE from 0.19 to 0.05 on Codex and to 0.11 on Claude Code, while achieving the lowest Brier score and highest BAS on both frameworks.
For discrimination, \methodnameshort{} is comparable to both verbalizers on Codex trajectories (AUROC 0.61 vs.\ 0.63) but weaker on Claude Code trajectories (0.54 vs.\ 0.64--0.67). 
This reinforces the distinction between calibration and discrimination: in this transfer setting, \methodnameshort{}s produce substantially better probability estimates and acceptance decisions despite weaker ranking quality.

\begin{table*}[t]
\centering
\caption{Transfer to other agent frameworks on SWE-Bench Verified.
Codex (GPT-5.5) has $n=305$ trajectories and 72.5\% success;
Claude Code (Claude Opus 4.7) has $n=304$ and 72.0\% success.
All estimators use Qwen-3.8 27B. ECE is adaptive ECE with 10 bins.}
\label{tab:other-frameworks}
\small
\setlength{\tabcolsep}{4pt}
\begin{tabular}{lcccc@{\hspace{1em}}cccc}
\toprule
& \multicolumn{4}{c}{Codex (GPT 5.5)} & \multicolumn{4}{c}{Claude Code (Opus 4.7)} \\
\cmidrule(lr){2-5}\cmidrule(l){6-9}
Estimator 
& ECE$\downarrow$ & Brier$\downarrow$ & AUROC$\uparrow$ & BAS$\uparrow$
& ECE$\downarrow$ & Brier$\downarrow$ & AUROC$\uparrow$ & BAS$\uparrow$ \\
\midrule
\textit{Qwen 3.8 27B} & \multicolumn{8}{c}{} \\

\quad Basic Verbalizer
& 0.19 & 0.24 & \textbf{0.63} & 0.22
& 0.19 & 0.23 & 0.64 & 0.25 \\
\quad Reason-as-Graph
& 0.19 & 0.23 & \textbf{0.63} & 0.23
& 0.19 & 0.23 & \textbf{0.67} & 0.26 \\
\rowcolor{blue!8}
\quad \textbf{CRG (ours)}
& \textbf{0.05} & \textbf{0.20} & 0.61 & \textbf{0.38}
& \textbf{0.11} & \textbf{0.22} & 0.54 & \textbf{0.34} \\
\bottomrule
\end{tabular}
\end{table*}

\subsection{Component Ablation}
\label{sec:analysis-component-ablation}

We examine the contribution of each component of our method in \autoref{tab:propagation-ablation}.
\textit{Reason-as-Graph} is our assurance-case inspired verbal baseline.
\textit{Reason-with-Graph} constructs the same claim and evidence graph as a \methodnameshort{}, but directly verbalizes confidence in the root claim $G_0$ conditioned on the graph, ablating our propagation method.
We find \textit{Reason-with-Graph} provides only a modest calibration improvement over Reason-as-Graph (adaptive ECE $0.38 \rightarrow 0.35$; Brier $0.37 \rightarrow 0.36$). Using our full method with leaf-level scoring and aggregation substantially improves both calibration and decision utility: ECE drops to $0.10$, Brier to $0.23$, and BAS increases from $-0.16$ to $0.17$. \textbf{Thus, the main gain does not come from simply exposing the model to a structured decomposition; it arises from using that decomposition to make localized confidence judgments and combine them into the final estimate.}
The improvement is primarily in confidence quality rather than ranking: \methodnameshort{} achieves AUROC $0.68$, below Reason-as-Graph ($0.73$) but above Reason-with-Graph ($0.61$). This is consistent with our main results, where improved calibration need not imply improved pairwise discrimination. Finally, \autoref{app:aggregation} compares alternative aggregation rules on the same graphs --- arithmetic and geometric means, maximum, the Fréchet bounds, and
the learned belief propagation of \citet{hou_probabilistic_2025} --- and finds that product aggregation gives the strongest calibration and competitive discrimination among training-free rules.

\subsection{Additional Findings}
\label{sec:results-additional-findings}

In \autoref{app:analysis-case-studies}, we analyze case studies, demonstrating how our claim graph might be audited by a human and overruled at inference time. 
Further analyses show that our method is cost-efficient (\autoref{app:analysis-cost}), our gains are also robust in post-hoc calibration with temperature scaling (\autoref{app:temperature-scaling}) and our approach is insensitive to LLM reasoning effort (\autoref{app:reasoning-effort}). 
In \autoref{sec:analysis-graph-quality}, we use textual entailment to test the refinement invariant in \autoref{eq:refinement-invariant} and diagnose violations of the independence approximation behind \autoref{eq:aggregation}; we find CRG remains effective despite imperfect adherence to these assumptions.

\begin{table}[t]
\centering
\caption{Ablating graph construction and confidence propagation on development data
(Qwen-3.8 27B, $n=495$).
Reason-as-Graph decomposes the task in the prompt and reports one confidence;
Reason-with-Graph builds the graph but scores the root directly rather than
propagating from leaves.}
\label{tab:propagation-ablation}
\small
\setlength{\tabcolsep}{4pt}
\begin{tabular}{lcccc}
\toprule
Method & ECE$\downarrow$ & Brier$\downarrow$ & AUROC$\uparrow$ & BAS$\uparrow$ \\
\midrule
Reason-as-Graph (no explicit graph)          & 0.38 & 0.37 & \textbf{0.73} & $-$0.23 \\
Reason-with-Graph (no propagation)  & 0.35 & 0.36 & 0.61 & $-$0.16  \\
\rowcolor{blue!8} \textbf{CRG (graph + propagation)} & \textbf{0.10} & \textbf{0.23} & 0.68 & \textbf{0.17}  \\
\bottomrule
\end{tabular}
\end{table}

\section{Conclusion}
\label{sec:conclusion}

We introduced \textbf{\methodnamefull{}}, which estimates task-success confidence from a single black-box agent trajectory by decomposing success into evidence-grounded claims and
aggregating their confidences. \textbf{\methodnameshort{}s}
improve calibration and decision utility across benchmarks over evaluated baselines
while retaining strong discrimination.
This enables practical, auditable confidence estimates for black-box agents without training.

\subsection*{AI use statement}

In this work, we used generative AI tools for refinement of the writing, formatting help, literature search, figure conversion and implementation, and in the assistance of writing for proofs.
We have not used generative AI tools for producing our ideas, designing our pipelines or experiments, and the rest of the required disclosure tasks are not applicable to this work.
Additionally, we used generative AI tools to write code and create plots, which we carefully reviewed and improved manually.
We have reviewed all AI-assisted content used anywhere in the production of this work. 
Each paper found through LLM-driven literature search has been carefully reviewed, and extensive further literature review has been conducted manually. 
We used LLMs for the writing in an atomic fashion: to refine the phrasing of individual points in the text, rather than their semantics.
We take responsibility for the final content of this work, including text, claims or artifacts produced with the aid of generative AI.

\subsection*{Ethics statement}

Our work aims to support safer use of LLM agents by providing confidence estimates for individual executions. However, these estimates should not be treated as guarantees of correctness. An inaccurate or overconfident estimate could create false reassurance, especially in high-stakes applications. 

Our method also inherits limitations from the LLMs used to construct and evaluate the reasoning graph: decompositions may be incomplete, dependencies between claims may be missed, and evidence may be misinterpreted. CRG confidence should therefore be used as a decision-support signal alongside appropriate verification and human oversight.

\subsection*{Reproducibility statement}

We support reproducibility by describing our method in \autoref{sec:methods} and the benchmarks, estimator settings, baselines, and metrics in \autoref{sec:experiments}. \autoref{app:datasets} documents benchmark filtering, trajectory collection, and success-label assignment; \autoref{app:prompts} provides the full prompts; and \autoref{app:detailed_proof} states the assumptions and proofs for confidence aggregation. The implementation repository contains the prompt templates, experiment configurations, evaluation scripts, and dependency lockfile used in our experiments. \autoref{app:paired-bootstrap} describes the resampling procedure for the uncertainty analysis.
Framework versions, randomization settings, and their scope are documented
in \autoref{app:reproducibility-details}.

\bibliography{iclr2027_conference}
\bibliographystyle{iclr2027_conference}

\appendix

\section{Appendix}
\subsection{Benchmarks and Trajectory Collection}
\label{app:datasets}

\subsubsection*{Test Set: Main Experiments}
We construct our test set for evaluating confidence estimators using three agentic benchmarks:
SWE-Bench Verified~\citep{chowdhury_introducing_2024},
SkillsBench~\citep{li_skillsbench_2026}, and
EnterpriseOps-Gym~\citep{malay_enterpriseops-gym_2026}.
Because confidence estimation is most informative when agent success is not
saturated, we focus on the more challenging tasks from each benchmark.
We describe the problems chosen and number of trajectories sampled below.

\paragraph{SWE-Bench Verified.}
SWE-Bench Verified evaluates software-engineering agents on real-world
repository issues.
We select the \texttt{test} split of
\texttt{princeton-nlp/SWE-bench\_Verified} and retain instances whose
\texttt{difficulty} annotation is greater than 15 minutes: \texttt{15 min - 1 hour},
\texttt{1-4 hours}, or \texttt{>4 hours}, yielding \textbf{306 problems}.

\paragraph{SkillsBench.}
SkillsBench contains 87 long-horizon tasks organized into three difficulty
levels: 6 \emph{Core} tasks requiring less than one hour, 53
\emph{Extended} tasks requiring 1--4 hours, and 28 \emph{Extreme} tasks
requiring more than four hours. We retain the Extended and Extreme tasks
and use the \emph{with-skills} condition, giving \textbf{81 eligible
problems}.
The with-skills condition loads the task's supplied skills into the agent
context. 

\paragraph{EnterpriseOps-Gym.}
EnterpriseOps-Gym evaluates agents on tool-heavy workflows spanning multiple
enterprise application domains. Unlike the other two benchmarks, it does not
provide the same human-effort difficulty annotation. We use \textbf{50\% of the problems within
each domain}, yielding \textbf{325 problems} while preserving domain
composition. The eight domains are Calendar, CSM, Drive, Email, HR, Hybrid,
ITSM, and Teams.

\paragraph{Trajectory collection.}
For each selected problem, we attempt one OpenHands
trajectory~\citep{wang2025openhandssoftwareagentsdk} using
GPT-5.5~\citep{openai2026gpt55},
Gemini-3.5 Flash~\citep{google2026gemini35flash}, or
MiniMax-M3~\citep{lai2026minimaxsparseattention}.
The saved run configurations use \emph{high} reasoning effort for all
three agent LLMs. 
SkillsBench and EnterpriseOps-Gym trajectories used
OpenHands SDK 1.31.0 with a 100-iteration cap. 
The SWE-Bench trajectories instead use a 500-iteration cap as they are sourced from the OpenHands Index \citep{openhandsindex2025}, rather than recomputed.
We omit any trajectories which crash or cannot be represented in the 262K context-window length of Qwen 3.8 27B for either \methodnameshort{}s or verbalized baselines.
This yields the following evaluation set counts: 900 SWE-Bench trajectories across 305
distinct problems (303 GPT-5.5, 300 Gemini-3.5 Flash, and 297 MiniMax-M3),
219 SkillsBench trajectories across 74 problems (73 per agent LLM), and
975 EnterpriseOps-Gym trajectories across 325 problems (325 per agent LLM).
These OpenHands trajectories constitute our main evaluation set. 
We separately evaluate transfer to Codex and Claude Code and describe that evaluation set in
\autoref{sec:results-other-frameworks}.

\paragraph{Success labels.}
For each scored trajectory, we use the benchmark's native verifier to
obtain the binary success label $\agentsuccess \in \{0,1\}$. For
SWE-Bench Verified and EnterpriseOps-Gym this is the \texttt{resolved} status in the
benchmark evaluation report. For SkillsBench, numeric verifier reward
counts as success only when it equals one; partial reward is a failure
for our binary evaluation. 

\subsubsection*{Development Set}

For ablations (\autoref{sec:analysis-component-ablation}, \autoref{app:depth}, \autoref{app:aggregation}), we construct a development set from two benchmarks:
\begin{enumerate}
    \item \textbf{SWE-Smith} \citep{yang2025swesmithscalingdatasoftware}: a dataset of \textit{synthetic} repository-level SWE problems. We sample 64 random problems.
    \item \textbf{EnterpriseOps Gym} \citep{malay_enterpriseops-gym_2026}: the tool-intensive benchmark used in our main evaluation, from which we sample 64 problems (8 per domain) which do not overlap with problems in the test set described above.
\end{enumerate}

We generate trajectories with a maximum of 100 steps, using the same agent LLMs: GPT 5.5, Gemini 3.5 Flash, and Minimax M3. 
While we do \textit{no hyper-parameter tuning} on this development set, it was used in the initial development of our prompts.\footnote{While our method has a maximum refinement depth hyper-parameter $k$, we find depths greater than 5 are never achieved, and therefore do not tune this parameter per-benchmark or to optimize performance on development data}
On the test set, we therefore demonstrate generalization of our method to harder, real problems from SWE-Bench Verified, and a new domain in SkillsBench.

\subsection{Calibration and Brier Score}
\label{app:metrics}

For $n$ scored trajectories, let $\agentconf_i\in[0,1]$ be the estimated
probability of success and $\agentsuccess_i\in\{0,1\}$ the observed outcome.
Given a partition of the trajectory indices into bins
$B_1,\ldots,B_M$, define
\[
\mathrm{conf}(B_m)=\frac{1}{|B_m|}\sum_{i\in B_m}\agentconf_i,
\qquad
\mathrm{acc}(B_m)=\frac{1}{|B_m|}\sum_{i\in B_m}\agentsuccess_i.
\]
Expected Calibration Error (ECE) \citep{guo_onconfcal_2017} is
\begin{equation}
\mathrm{ECE}
=\sum_{m=1}^{M}\frac{|B_m|}{n}
\left|\mathrm{acc}(B_m)-\mathrm{conf}(B_m)\right|.
\label{eq:ece}
\end{equation}
For our main results, we use adaptive ECE \citep{Nixon_2019_CVPR_Workshops}:
we sort trajectories by confidence and divide them into $M=10$
approximately equal-count bins. In our implementation, bin sizes differ
by at most one, and tied confidences retain their original order according to a stable sort, following \cite{Nixon_2019_CVPR_Workshops}.
Equal-count binning avoids sparsely populated bins, especially on smaller datasets, reducing the
variance of per-bin accuracy estimates. It has also been reported to have
lower bias \citep{pmlr-v151-roelofs22a}.
\autoref{app:ece-bin-sweep} examines other bin counts and equal-width bins.

We also report the binary Brier score,
\begin{equation}
\mathrm{Brier}
=\frac{1}{n}\sum_{i=1}^{n}
(\agentconf_i-\agentsuccess_i)^2.
\label{eq:brier}
\end{equation}
Lower values indicate better probabilistic predictions for both metrics.
The Brier score summarizes squared prediction error, whereas ECE
specifically measures the gap between confidence and observed success
frequency within bins.

\subsection{Case Studies}
\label{app:analysis-case-studies}

\begin{figure*}[t]
\centering
\begin{minipage}[t]{0.49\textwidth}
\centering
\resizebox{\linewidth}{!}{%
\begin{tikzpicture}[
  >=Stealth,
  edge/.style={->, draw=gray!70, line width=0.8pt},
  base/.style={rounded corners=2pt, align=center, minimum height=10mm, inner sep=4pt, font=\small},
  high/.style={base, draw=crghighborder, fill=crghighfill},
  mid/.style={base, draw=crgmidborder, fill=crgmidfill}
]
\node[mid, minimum width=52mm] (root) {\textbf{Successful final outcome}\\aggregate: 0.58};
\node[mid, below left=16mm and 26mm of root, minimum width=39mm] (obj) {\textbf{Primary objective}\\conf 0.71};
\node[high, below=16mm of root, minimum width=39mm] (int) {\textbf{Integrity}\\conf 0.90};
\node[high, below right=16mm and 26mm of root, minimum width=39mm] (pol) {\textbf{Policy compliance}\\conf 0.90};

\node[high, below=19mm of obj, xshift=-30mm, minimum width=28mm] (case) {\textbf{Case created}\\conf 0.96};
\node[high, right=5mm of case, minimum width=28mm] (casestatus) {\textbf{Case status}\\conf 0.93};
\node[high, right=5mm of casestatus, minimum width=28mm] (task) {\textbf{Task created}\\conf 0.95};
\node[high, right=5mm of task, minimum width=28mm] (survey) {\textbf{Survey linked}\\conf 0.93};
\node[high, below=11mm of case, minimum width=28mm] (taskstatus) {\textbf{Task status}\\conf 0.95};
\node[high, below=11mm of casestatus, minimum width=34mm] (notify) {\textbf{Notification sent}\\conf 0.95};

\draw[edge] (root) -- (obj);
\draw[edge] (root) -- (int);
\draw[edge] (root) -- (pol);
\draw[edge] (obj) -- (case);
\draw[edge] (obj) -- (casestatus);
\draw[edge] (obj) -- (task);
\draw[edge] (obj) -- (survey);
\draw[edge] (case) -- (taskstatus);
\draw[edge] (casestatus) -- (notify);
\end{tikzpicture}%
}
\par\smallskip\textbf{(a) Successful \methodnameshort{ }Estimation}
\end{minipage}\hfill
\begin{minipage}[t]{0.49\textwidth}
\centering
\resizebox{\linewidth}{!}{%
\begin{tikzpicture}[
  >=Stealth,
  edge/.style={->, draw=gray!70, line width=0.8pt},
  badedge/.style={->, draw=crglowborder, line width=1pt},
  base/.style={rounded corners=2pt, align=center, minimum height=10mm, inner sep=4pt, font=\small},
  high/.style={base, draw=crghighborder, fill=crghighfill},
  low/.style={base, draw=crglowborder, fill=crglowfill}
]
\node[low, minimum width=52mm] (root) {\textbf{Successful final outcome}\\aggregate: 0.08};
\node[high, below left=16mm and 24mm of root, minimum width=38mm] (obj) {\textbf{Objective achieved}\\conf 0.88};
\node[high, below=16mm of root, minimum width=38mm] (int) {\textbf{Data integrity}\\conf 0.93};
\node[low, below right=16mm and 24mm of root, minimum width=38mm] (pol) {\textbf{Policy compliance}\\conf 0.10};

\node[high, below left=18mm and 3mm of obj, minimum width=28mm] (case) {\textbf{Case created}\\conf 0.95};
\node[high, below right=18mm and 3mm of obj, minimum width=28mm] (interaction) {\textbf{Interaction}\\conf 0.93};

\node[low, below=30mm of pol, text width=57mm] (ev) {\textbf{Undermining evidence}\\First attempt used status \texttt{completed}, outside the allowed enum domain; the system rejected the operation.};

\draw[edge] (root) -- (obj);
\draw[edge] (root) -- (int);
\draw[edge] (root) -- (pol);
\draw[edge] (obj) -- (case);
\draw[edge] (obj) -- (interaction);
\draw[badedge] (ev) -- node[right, font=\scriptsize, text=crglowborder]{Undermines} (pol);
\end{tikzpicture}%
}
\par\smallskip\textbf{(b) Failed \methodnameshort{} Estimation}
\end{minipage}
\caption{\methodnameshort{}s for two successful EnterpriseOps-Gym
trajectories. (a) A correct estimate: an HR-enrollment task decomposes into
verifiable subtasks (case created, status set, survey linked, notification
sent), each with high confidence. (b) An incorrect estimate, caused by
incomplete graph construction: the agent's first attempt used a status outside
the allowed enum domain and was rejected, but no evidence item records the
subsequent correction. The policy-compliance claim is therefore scored against
the violation alone, receiving $0.10$ and driving the root to $0.08$. Because
the graph exposes which claim and which cited step produced the estimate, a
user can inspect this reasoning and overrule it.}
\label{fig:case-studies}
\end{figure*}
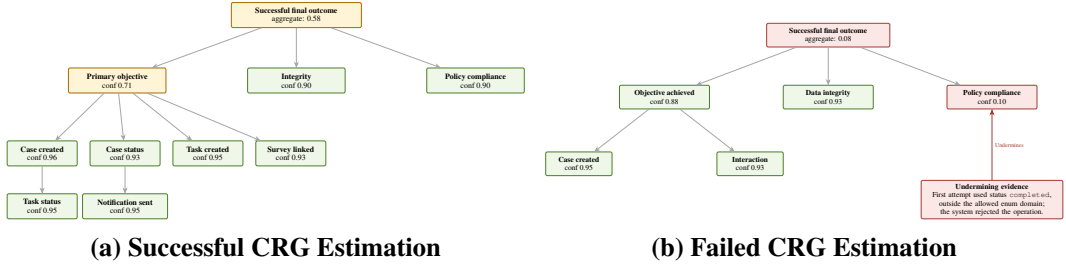

We examine two EnterpriseOps-Gym trajectories to illustrate how
\methodnameshort{} connects task-specific claims and trajectory
evidence to the final confidence estimate.

\paragraph{Aggregating high-confidence conditions.}
In \autoref{fig:case-studies}(a), a health-insurance enrollment task
requires creating an HR case, setting its status, creating and linking
a survey task, updating the task status, and notifying the employee.
The corresponding leaf claims receive confidences of $0.93$--$0.96$.
Their aggregation gives a primary-objective confidence of $0.71$ and,
after accounting for integrity and policy-compliance claims, a root
confidence of $0.58$. \textbf{The graph makes explicit why several
individually high-confidence conditions can still yield a more cautious
estimate of overall task success.}

\paragraph{Tracing a low-confidence estimate.}
\autoref{fig:case-studies}(b) shows a trajectory that also succeeded, but on
which \methodnameshort{} estimation is wrong.
The agent initially attempts to create a support case using the status
\texttt{completed}, which is outside the allowed value set and is rejected by
the system.
\methodnameshort{} attaches this event as undermining evidence to the
policy-compliance claim, which receives confidence $0.10$; although the
case-creation and interaction claims receive $0.95$ and $0.93$, the root falls
to $0.08$.
The agent subsequently corrected the value and completed the task, but no
evidence item records the correction, so the claim was scored against the
violation alone --- a failure of graph construction rather than of the
confidence estimate it supports.
\textbf{Unlike a single scalar confidence, the graph identifies the specific
claim and trajectory evidence responsible for the low estimate}, so a user
deciding whether to accept this solution can check that step, find the
correction, and overrule the estimate.

\subsection{Cost}
\label{app:analysis-cost}

We compare the inference cost and calibration of each estimator in \autoref{fig:cost}. Because methods differ in model size and in prompt and completion length, we convert inference usage to an estimated cost per trajectory in USD.\footnote{We run Qwen-3.8 27B locally and estimate its cost using the rates reported by \href{https://openrouter.ai/qwen/qwen3.8-27b}{OpenRouter}.}

\textbf{Across all three benchmarks, \methodnameshort{} lies on the cost--calibration Pareto frontier, together with the cheaper but less calibrated Surrogate LNSP.} \methodnameshort{} costs approximately \$0.04--\$0.07 per trajectory and achieves lower adaptive ECE than the surrogate on every benchmark, with the largest gap on EnterpriseOps-Gym (0.13 vs.\ 0.22). Thus, \methodnameshort{} trades additional inference cost for substantially improved calibration in settings where the cheaper surrogate is insufficient.

The gains are also not explained by inference budget alone. On SWE-Bench Verified, increasing the reasoning effort of verbalized baselines can make them more expensive than \methodnameshort{} while leaving them substantially less calibrated. This suggests that the improvement comes from how \methodnameshort{} structures and aggregates confidence estimation, rather than simply from additional test-time compute.

\begin{figure}[t]
\centering
\includegraphics[width=\linewidth]{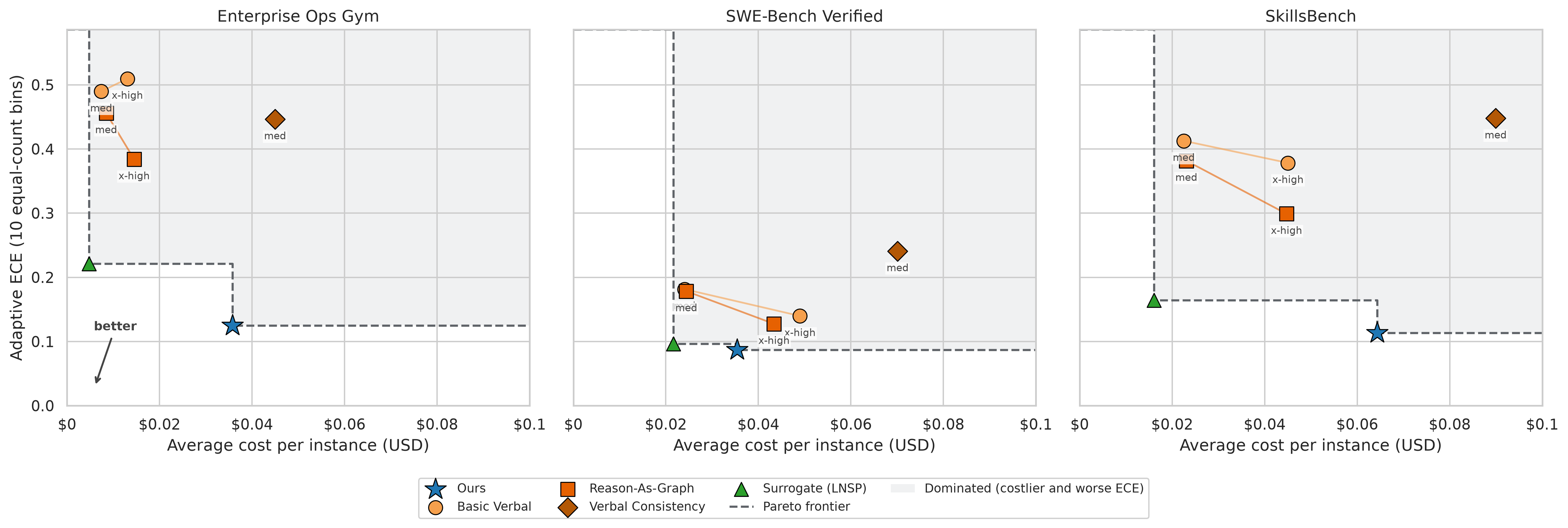}
\caption{Cost--calibration tradeoff of each estimator: average cost per trajectory (USD) against adaptive ECE ($\downarrow$), so lower-left is better. \plotours{\methodnameshort{}s} are the best-calibrated estimators on every benchmark, sharing the Pareto frontier with only the cheaper but worse-calibrated \plotsurrogate{Surrogate LNSP}. Importantly, \plotours{\methodnameshort{}s} are cost-efficient in absolute terms: \$0.04--0.07 per trajectory, compared to the cost of generating the trajectory being judged, which averages \$0.15, \$1.06, and \$0.95 on EnterpriseOps-Gym, SWE-Bench, and SkillsBench respectively.
}
\label{fig:cost}
\end{figure}

\subsection{ECE Robustness to Binning}
\label{app:ece-bin-sweep}

ECE depends on how predictions are partitioned into bins, so we recompute
standard and adaptive ECE with 5, 8, 10, 15, and 20 bins.
\autoref{tab:ece-bin-sweep} shows that the calibration ranking is robust to
these choices. CRG with Qwen-3.8 27B achieves the lowest or tied-lowest
adaptive ECE for every benchmark and bin count. It also achieves the lowest
or tied-lowest standard ECE throughout, except for the five-bin SWE-Bench
setting, where its ECE differs from Surrogate LNSP by 0.01. Thus, neither the
ten-bin choice nor equal-mass binning explains CRG's calibration gains.

\begin{table*}[t]
\centering
\caption{Sensitivity of confidence-estimation calibration error to the number of bins. ECE uses equal-width bins and Adaptive ECE uses equal-mass bins. Best per column in \textbf{bold}, second \underline{underlined}. $^\dagger$White-box: uses token probabilities from a surrogate LLM.}
\label{tab:ece-bin-sweep}
\setlength{\tabcolsep}{3pt}
\resizebox{\textwidth}{!}{%
\begin{tabular}{ll ccccc ccccc ccccc}
\toprule
& & \multicolumn{5}{c}{SWE-Bench Verified} & \multicolumn{5}{c}{EnterpriseOps-Gym} & \multicolumn{5}{c}{SkillsBench} \\
\cmidrule(lr){3-7} \cmidrule(lr){8-12} \cmidrule(lr){13-17}
Estimator & LLM & 5 & 8 & 10 & 15 & 20 & 5 & 8 & 10 & 15 & 20 & 5 & 8 & 10 & 15 & 20 \\
\midrule
\multicolumn{17}{l}{\textit{Standard ECE (equal-width bins)}} \\
\addlinespace[2pt]
\multirow{2}{*}{Basic Verbalizer} & GPT-5.6 Sol & 0.21 & 0.21 & 0.21 & 0.21 & 0.21 & 0.58 & 0.58 & 0.58 & 0.58 & 0.58 & 0.46 & 0.46 & 0.46 & 0.46 & 0.46 \\
 & Qwen-3.8 27B & 0.18 & 0.18 & 0.18 & 0.18 & 0.18 & 0.49 & 0.49 & 0.49 & 0.49 & 0.49 & 0.41 & 0.41 & 0.41 & 0.41 & 0.41 \\
\multirow{2}{*}{Reason-as-Graph} & GPT-5.6 Sol & 0.17 & 0.17 & 0.17 & 0.17 & 0.18 & 0.57 & 0.57 & 0.57 & 0.57 & 0.57 & 0.43 & 0.43 & 0.43 & 0.44 & 0.44 \\
 & Qwen-3.8 27B & 0.18 & 0.18 & 0.18 & 0.18 & 0.18 & 0.46 & 0.46 & 0.46 & 0.46 & 0.46 & 0.38 & 0.38 & 0.38 & 0.38 & 0.38 \\
Verbal Consistency & Qwen-3.8 27B & 0.25 & 0.25 & 0.25 & 0.25 & 0.25 & 0.45 & 0.45 & 0.45 & 0.45 & 0.45 & 0.45 & 0.45 & 0.45 & 0.45 & 0.45 \\
Surrogate LNSP$^\dagger$ & Qwen-3.8 27B & \textbf{0.07} & \underline{0.09} & \underline{0.09} & \underline{0.09} & \textbf{0.09} & \underline{0.22} & \underline{0.22} & \underline{0.22} & \underline{0.22} & \underline{0.22} & \underline{0.11} & \underline{0.13} & \underline{0.13} & \underline{0.15} & \underline{0.17} \\
\midrule
\rowcolor{blue!8} & GPT-5.6 Sol & 0.16 & 0.15 & 0.16 & 0.16 & \underline{0.16} & 0.25 & 0.25 & 0.25 & 0.25 & 0.25 & 0.22 & 0.23 & 0.22 & 0.23 & 0.24 \\
\rowcolor{blue!8} \multirow{-2}{*}{\textbf{CRG (ours)}} & Qwen-3.8 27B & \underline{0.08} & \textbf{0.08} & \textbf{0.08} & \textbf{0.08} & \textbf{0.09} & \textbf{0.13} & \textbf{0.12} & \textbf{0.13} & \textbf{0.13} & \textbf{0.13} & \textbf{0.10} & \textbf{0.11} & \textbf{0.12} & \textbf{0.13} & \textbf{0.13} \\
\midrule
\multicolumn{17}{l}{\textit{Adaptive ECE (equal-mass bins)}} \\
\addlinespace[2pt]
\multirow{2}{*}{Basic Verbalizer} & GPT-5.6 Sol & 0.20 & 0.20 & 0.20 & 0.20 & 0.20 & 0.58 & 0.58 & 0.58 & 0.58 & 0.58 & 0.46 & 0.46 & 0.46 & 0.46 & 0.46 \\
 & Qwen-3.8 27B & 0.18 & 0.18 & 0.18 & 0.18 & 0.18 & 0.49 & 0.49 & 0.49 & 0.49 & 0.49 & 0.41 & 0.41 & 0.41 & 0.41 & 0.41 \\
\multirow{2}{*}{Reason-as-Graph} & GPT-5.6 Sol & 0.17 & 0.17 & 0.17 & 0.17 & 0.17 & 0.57 & 0.57 & 0.57 & 0.57 & 0.57 & 0.43 & 0.43 & 0.43 & 0.43 & 0.43 \\
 & Qwen-3.8 27B & 0.18 & 0.18 & 0.18 & 0.18 & 0.18 & 0.46 & 0.46 & 0.46 & 0.46 & 0.46 & 0.38 & 0.38 & 0.38 & 0.38 & 0.38 \\
Verbal Consistency & Qwen-3.8 27B & 0.24 & 0.24 & 0.24 & 0.24 & 0.24 & 0.45 & 0.42 & 0.45 & 0.45 & 0.45 & 0.45 & 0.45 & 0.45 & 0.45 & 0.45 \\
Surrogate LNSP$^\dagger$ & Qwen-3.8 27B & \textbf{0.09} & \underline{0.09} & \underline{0.10} & \underline{0.10} & \underline{0.10} & \underline{0.22} & \underline{0.22} & \underline{0.22} & \underline{0.22} & \underline{0.22} & \underline{0.16} & \underline{0.16} & \underline{0.16} & \underline{0.17} & \underline{0.17} \\
\midrule
\rowcolor{blue!8} & GPT-5.6 Sol & \underline{0.16} & 0.15 & 0.16 & 0.16 & 0.16 & 0.24 & 0.24 & 0.24 & 0.24 & 0.25 & 0.22 & 0.21 & 0.22 & 0.24 & 0.23 \\
\rowcolor{blue!8} \multirow{-2}{*}{\textbf{CRG (ours)}} & Qwen-3.8 27B & \textbf{0.09} & \textbf{0.08} & \textbf{0.09} & \textbf{0.09} & \textbf{0.09} & \textbf{0.13} & \textbf{0.12} & \textbf{0.13} & \textbf{0.13} & \textbf{0.13} & \textbf{0.10} & \textbf{0.11} & \textbf{0.11} & \textbf{0.15} & \textbf{0.14} \\
\bottomrule
\end{tabular}
}
\end{table*}

\subsection{Maximum Refinement Depth}
\label{app:depth}

We define the refinement depth of a claim as the number of refinement edges
on the path from the root to that claim, and constrain the maximum
root-to-leaf refinement depth to $k$.
We use a single value, $k=5$, for every benchmark rather than tuning
refinement depth separately for each setting.

To examine whether this limit constrains graph construction, we vary
$k \in \{3,4,5,6,7,8\}$ on the development set
and measure the maximum root-to-leaf refinement depth of the
graphs that are actually produced.
As shown in \autoref{fig:depth}, the observed maximum depth increases from
3 to 4 to 5 as $k$ increases from 3 to 5, but remains at 5 when the allowed
maximum is increased further to 6, 7, or 8.
The 95th-percentile graph depth reaches 4 at $k=4$ and does not increase for
larger values of $k$.
Thus, $k=5$ permits the occasional depth-5 graph while larger depth budgets
do not lead the constructor to produce deeper graphs.
We therefore use $k=5$ throughout our experiments.

\begin{figure}[h]
\centering
\includegraphics[width=0.7\linewidth]{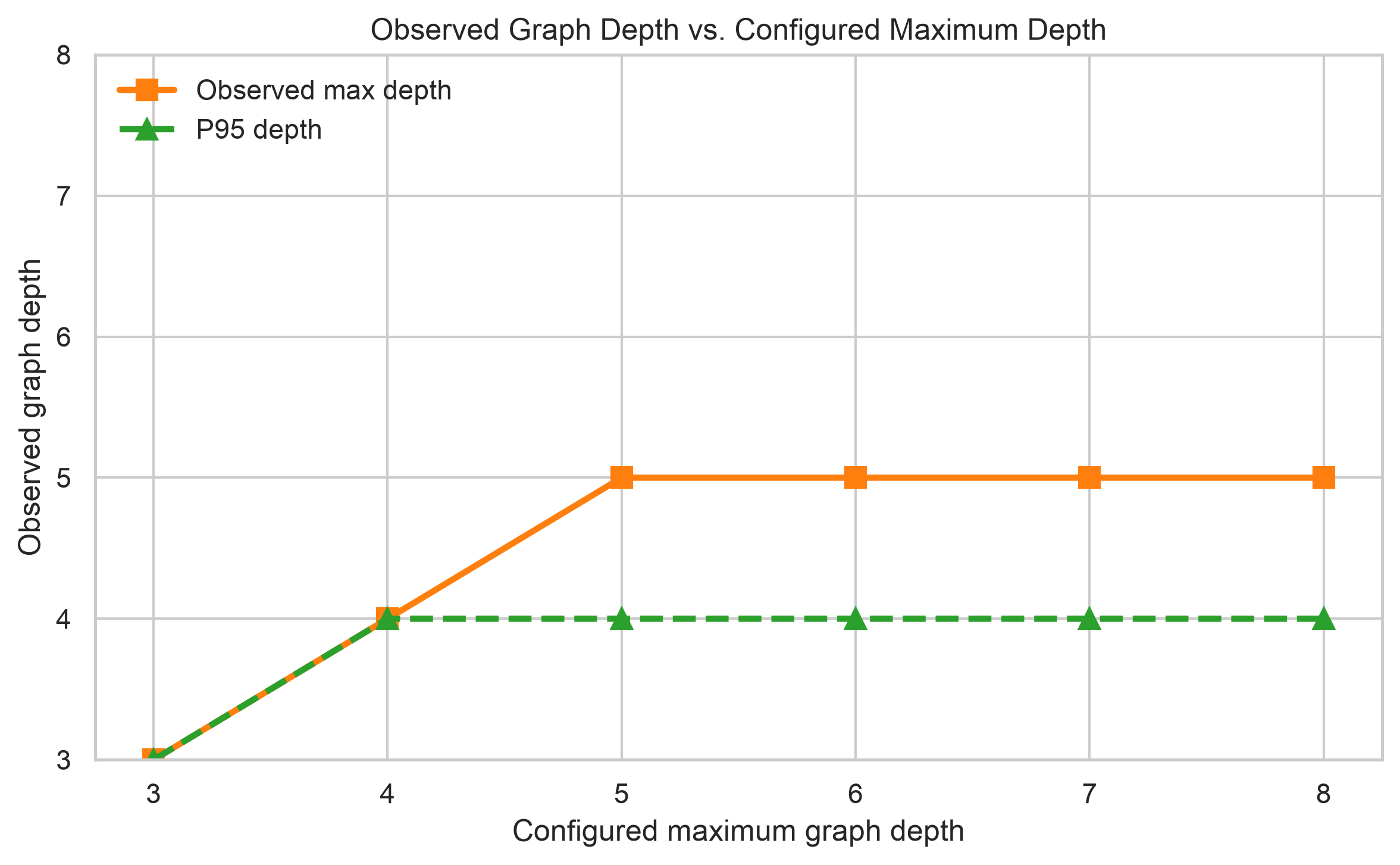}
\caption{Observed refinement depth vs.\ configured maximum refinement depth $k$
on development data. Maximum and 95th-percentile observed depth stop growing
at $k=5$ and $k=4$, so we choose $k=5$ throughout our main experiments.}
\label{fig:depth}
\end{figure}

\subsection{Estimator Reasoning Effort}
\label{app:reasoning-effort}

We additionally test the sensitivity of CRG estimation to the reasoning
effort allocated to Qwen-3.8 27B in both graph construction and leaf
confidence estimation.
Across \emph{low}, \emph{medium}, and \emph{xhigh} reasoning effort,
predictive performance is nearly unchanged:
reported ECE is $0.08$, $0.09$, and $0.08$ respectively;
Brier score is $0.19$ at all three settings; and AUROC is
$0.78$, $0.77$, and $0.79$ (\autoref{tab:reasoning-effort}).
In contrast, the reported generated-token count increases from approximately
5.3K at low effort to 6.1K at medium effort and 10.4K at xhigh effort.
We therefore use \emph{medium} reasoning effort in the main experiments,
which avoids the substantially higher inference cost of xhigh reasoning
as we observe no corresponding improvement in confidence quality.

\begin{table}[h]
    \centering
    \caption{Sensitivity of the Qwen-3.8 27B CRG estimator
    to reasoning effort in both graph construction and leaf scoring.}
    \label{tab:reasoning-effort}
    \begin{tabular}{lrrrr}
        \toprule
        Effort & ECE & Brier & AUROC & Generated tokens \\
        \midrule
        Low & 0.08 & 0.19 & 0.78 & 5.3K \\
        Medium & 0.09 & 0.19 & 0.77 & 6.1K \\
        Xhigh & 0.08 & 0.19 & 0.79 & 10.4K \\
        \bottomrule
    \end{tabular}
\end{table}

\subsection{Surrogate Log-Probability Aggregations}
\label{app:surrogate-aggregations}

Following \citet{bouchard_beyond_2026}, we compare several ways to aggregate action token probabilities from the surrogate LLM to form a confidence estimate. Let action $a_t=(w_{t,1},\ldots,w_{t,n_t})$ contain $n_t$ tokens, and let $h_t$ denote the reconstructed LLM context preceding that action. 
Note that for many proprietary frontier models, this reconstructed LLM context excludes the reasoning tokens present in the original LLM context, which are not retrievable.
The surrogate assigns each \textit{action step} the log probability
\begin{equation}
    \ell_t
    =
    \sum_{j=1}^{n_t}
    \log p_{\mathrm{surr}}(w_{t,j}\mid h_t,w_{t,<j}).
\end{equation}
From this quantity, we define the sequence probability (SP) and length-normalized sequence probability (LNSP) of an action step, respectively, as
\begin{equation}
    s_t = \exp(\ell_t),
    \qquad
    r_t = \exp\!\left(\frac{\ell_t}{n_t}\right).
\end{equation}

Sequence probability (SP) multiplies the token probabilities within an action. Length-normalized sequence probability (LNSP) takes their geometric mean, controlling for action length.

For a trajectory $\agenttraj$ containing $T$ scored actions, we evaluate the following trajectory-level confidence estimates:
\begin{itemize}
    \item \textbf{LNSP (first)} and \textbf{LNSP (last)} return $r_1$ and
    $r_T$, the LNSP of the first and last actions.
    \item \textbf{LNSP (mean)} returns the arithmetic mean
    $T^{-1}\sum_{t=1}^{T}r_t$.
    \item \textbf{LNSP (min)} returns the minimum: $\min_{1\leq t\leq T}r_t$.
    \item \textbf{SP (first)} and \textbf{SP (last)} return $s_1$ and
    $s_T$, the unnormalized sequence probabilities of the first and last actions, respectively.
    \item \textbf{SP (mean)} computes each action's
non-length-normalized sequence probability and averages across actions:
$T^{-1}\sum_{t=1}^{T} s_t$.
    \item \textbf{SP (min)} again returns the minimum:
    $\min_{1\leq t\leq T}s_t$.
\end{itemize}
The Surrogate LNSP baseline in the main experiments (\autoref{tab:main-results}) is \textbf{LNSP (last)}.
All variants use Qwen-3.8 27B as the surrogate LLM and reconstruct its context
at each scored agent action.

\begin{table*}[t]
\centering
\caption{Surrogate log-probability aggregation results on OpenHands trajectories, averaged over agent LLMs. We report adaptive ECE (10 bins), Brier Score, AUROC, BAS, mean confidence, and confidence standard deviation. Best per ECE, Brier, AUROC, and BAS column in \textbf{bold}, second \underline{underlined}. $^\dagger$White-box: uses token probabilities from a surrogate LLM.}
\label{tab:surrogate-aggregation-results}
\setlength{\tabcolsep}{3pt}
\resizebox{\textwidth}{!}{%
\begin{tabular}{ll cccccc cccccc cccccc}
\toprule
& & \multicolumn{6}{c}{SWE-Bench Verified} & \multicolumn{6}{c}{EnterpriseOps-Gym} & \multicolumn{6}{c}{SkillsBench} \\
\cmidrule(lr){3-8} \cmidrule(lr){9-14} \cmidrule(lr){15-20}
Estimator & LLM & ECE$\downarrow$ & Brier$\downarrow$ & AUROC$\uparrow$ & BAS$\uparrow$ & Mean Conf. & Std. Conf. & ECE$\downarrow$ & Brier$\downarrow$ & AUROC$\uparrow$ & BAS$\uparrow$ & Mean Conf. & Std. Conf. & ECE$\downarrow$ & Brier$\downarrow$ & AUROC$\uparrow$ & BAS$\uparrow$ & Mean Conf. & Std. Conf. \\
\midrule
Action LNSP (first)$^\dagger$ & Qwen-3.8 27B & 0.30 & 0.30 & 0.48 & 0.28 & 0.43 & 0.08 & \underline{0.22} & 0.31 & 0.49 & $-$0.02 & 0.61 & 0.15 & \underline{0.14} & 0.27 & 0.52 & \underline{0.13} & 0.41 & 0.11 \\
Action LNSP (last)$^\dagger$ & Qwen-3.8 27B & \underline{0.10} & \underline{0.21} & 0.47 & \underline{0.36} & 0.69 & 0.07 & \underline{0.22} & 0.28 & 0.61 & 0.04 & 0.61 & 0.09 & 0.16 & 0.28 & 0.46 & 0.09 & 0.57 & 0.13 \\
Action LNSP (mean)$^\dagger$ & Qwen-3.8 27B & \underline{0.10} & \textbf{0.20} & 0.48 & \textbf{0.37} & 0.79 & 0.04 & 0.35 & 0.36 & 0.60 & $-$0.09 & 0.74 & 0.09 & 0.23 & 0.30 & 0.48 & 0.07 & 0.72 & 0.04 \\
Action LNSP (min)$^\dagger$ & Qwen-3.8 27B & 0.34 & 0.32 & 0.51 & 0.26 & 0.39 & 0.08 & \textbf{0.13} & \underline{0.26} & 0.56 & \underline{0.07} & 0.53 & 0.12 & 0.15 & \underline{0.26} & \underline{0.58} & \textbf{0.14} & 0.37 & 0.09 \\
Action SP (first)$^\dagger$ & Qwen-3.8 27B & 0.73 & 0.73 & 0.49 & 0.00 & 0.00 & 0.00 & 0.39 & 0.39 & 0.58 & 0.00 & 0.00 & 0.00 & 0.49 & 0.49 & 0.53 & 0.00 & 0.00 & 0.00 \\
Action SP (last)$^\dagger$ & Qwen-3.8 27B & 0.73 & 0.73 & 0.47 & 0.00 & 0.00 & 0.00 & 0.39 & 0.39 & \underline{0.62} & 0.00 & 0.00 & 0.00 & 0.49 & 0.49 & 0.51 & 0.00 & 0.00 & 0.00 \\
Action SP (mean)$^\dagger$ & Qwen-3.8 27B & 0.73 & 0.72 & 0.48 & 0.01 & 0.01 & 0.01 & 0.39 & 0.39 & 0.58 & 0.00 & 0.01 & 0.02 & 0.49 & 0.49 & 0.46 & 0.00 & 0.00 & 0.01 \\
Action SP (min)$^\dagger$ & Qwen-3.8 27B & 0.73 & 0.73 & 0.54 & 0.00 & 0.00 & 0.00 & 0.39 & 0.39 & 0.61 & 0.00 & 0.00 & 0.00 & 0.49 & 0.49 & 0.56 & 0.00 & 0.00 & 0.00 \\
\midrule
\rowcolor{blue!8}  & GPT-5.6 Sol & 0.16 & 0.23 & \textbf{0.62} & \underline{0.36} & 0.59 & 0.25 & 0.24 & 0.29 & \textbf{0.67} & 0.06 & 0.15 & 0.20 & 0.22 & 0.29 & \textbf{0.65} & 0.09 & 0.40 & 0.34 \\
\rowcolor{blue!8} \multirow{-2}{*}{\textbf{CRG (ours)}} & Qwen-3.8 27B & \textbf{0.09} & \textbf{0.20} & \underline{0.57} & \textbf{0.37} & 0.70 & 0.16 & \textbf{0.13} & \textbf{0.24} & \underline{0.62} & \textbf{0.08} & 0.39 & 0.24 & \textbf{0.11} & \textbf{0.24} & \textbf{0.65} & \textbf{0.14} & 0.48 & 0.24 \\
\bottomrule
\end{tabular}
}
\end{table*}

\paragraph{Findings.} \autoref{tab:surrogate-aggregation-results} presents our complete comparison with surrogate methods.
\methodnameshort{}s remain the strongest overall method, achieving the best ECE, Brier score, AUROC, and BAS across benchmarks.
For the tasks we evaluated, sequence probability (SP) aggregation provided nearly useless estimates: harder agentic tasks yield longer actions which were at least somewhat uncertain in the surrogate model and earned confidence estimates consistently near zero.
For length-normalized approaches (LNSP), we find that the ranking among surrogate methods varies across benchmarks.
\textbf{Action LNSP (last)} is most consistent, performing best among surrogates on SWE-Bench Verified, and competitively elsewhere. 
While \textbf{Action LNSP (min)} performs best on EnterpriseOps-Gym, it performs substantially worse on SWE-Bench Verified (ECE 0.34 vs. 0.10 for Action LNSP (last)). \textbf{Action LNSP (first)} is best on SkillsBench, but inconsistent elsewhere.
Finally, we find that all surrogate methods perform poorly as discriminators, matching our findings in \autoref{sec:results}.

\subsection{Aggregation Rules}
\label{app:aggregation}

Our default aggregation computes root confidence as the product of the
leaf-claim confidences.
If the leaf claims form an exact conjunction and are conditionally independent
given the observed trajectory, this product is equal to their joint
probability.
In practice, conditional independence cannot be assumed: dependencies
between leaf claims may remain even after decomposition.
We therefore view the product as a simple approximation induced by the
conjunctive structure, and empirically compare it with alternative aggregation
rules.
The aggregation methods we evaluate are as follows:
Arithmetic and geometric means do not preserve the conjunctive semantics:
high confidence in most leaves can compensate for low confidence in a required
condition, and their output need not decrease as additional necessary
conditions are introduced.
The maximum is more permissive still, retaining only the strongest leaf.
In the two geometric-mean variants, one computes a geometric mean at each
internal goal from its direct children (propagating), whereas the other computes a single
geometric mean over all goal leaves (leaves).
When dependence among the leaf claims is unknown, their marginal confidences
identify only the Fr\'echet bounds on the probability of their conjunction:
\[
\max\!\left\{0,\sum_{G\in\leafset}c_G-(|\leafset|-1)\right\}
\;\leq\;
P\!\left(\bigcap_{G\in\leafset}G\,\middle|\,\agenttraj\right)
\;\leq\;
\min_{G\in\leafset}c_G.
\]
The union-bound and minimum rules take the lower and upper endpoints,
respectively. These bounds permit arbitrary dependence but yield only extreme
estimates; the product gives a point estimate by approximating the leaf claims
as conditionally independent.

We also compare against the probabilistic tree model of
\citet{hou_probabilistic_2025}.
Their learned emission probabilities use held-out truth-labeled data, however,
and therefore fall outside our training-free setting; \textbf{we include this variant
as a diagnostic comparison rather than as a directly admissible baseline.}

In the aggregation ablation (\autoref{tab:aggregation-rules}), the
goal-leaf product has the lowest ECE and Brier score among the compared
rules for EnterpriseOps-Gym by far. It also has the best BAS, and strong discrimination. On SWE-Smith, the learned emissions of \cite{hou_probabilistic_2025} are comparably calibrated, though this does not generalize to the other benchmark. The lower Fr\'echet bound also performs well, but the product remains stronger overall, supporting both conjunctive aggregation and our conditional-independence approximation.

\begin{table*}[t]
\centering
\caption{Comparing rules for combining leaf confidences into root confidence on development data (Qwen-3.8 27B). All rows use the same graphs, so differences isolate the aggregation rule. ECE denotes adaptive ECE with 10 bins.}
\label{tab:aggregation-rules}
\setlength{\tabcolsep}{3pt}
\resizebox{\textwidth}{!}{%
\begin{tabular}{lcccc cccc}
\toprule
& \multicolumn{4}{c}{EnterpriseOps-Gym ($n=192$)} & \multicolumn{4}{c}{SWE-Smith ($n=187$)} \\
\cmidrule(lr){2-5} \cmidrule(lr){6-9}
Aggregation rule & ECE$\downarrow$ & Brier$\downarrow$ & AUROC$\uparrow$ & BAS$\uparrow$ & ECE$\downarrow$ & Brier$\downarrow$ & AUROC$\uparrow$ & BAS$\uparrow$ \\
\midrule
Arithmetic mean                     & 0.51 & 0.47 & 0.71 & $-$0.43 & 0.25 & 0.28 & \textbf{0.57} & 0.08 \\
Geometric mean (propagating)        & 0.46 & 0.42 & 0.71 & $-$0.35 & 0.25 & 0.28 & \textbf{0.57} & 0.08 \\
Geometric mean (leaves)             & 0.45 & 0.41 & \textbf{0.72} & $-$0.33 & 0.25 & 0.28 & \textbf{0.57} & 0.08 \\
Maximum                             & 0.59 & 0.58 & 0.51 & $-$1.10 & 0.29 & 0.30 & 0.54 & $-$0.10 \\
Upper Fr\'echet bound (minimum)     & 0.28 & 0.29 & \textbf{0.72} & $-$0.07 & 0.20 & 0.26 & 0.56 & 0.17 \\
Lower Fr\'echet bound (union bound) & 0.22 & 0.25 & 0.64 & 0.06 & 0.14 & \textbf{0.22} & 0.55 & 0.29 \\
\citet{hou_probabilistic_2025}: \underline{learned emissions}       & 0.25 & 0.28 & 0.66 & 0.01 & \textbf{0.07} & \textbf{0.22} & \textbf{0.57} & \textbf{0.31} \\
\rowcolor{blue!8} Product (ours)    & \textbf{0.13} & \textbf{0.22} & 0.67 & \textbf{0.08} & 0.11 & \textbf{0.22} & 0.55 & 0.29 \\
\bottomrule
\end{tabular}%
}
\end{table*}

\subsection{Temperature-Scaled Baselines}
\label{app:temperature-scaling}

Our target setting is training-free and therefore does not permit fitting
post-hoc calibration parameters on labeled trajectories. Nevertheless, we evaluate whether our gains persist after post-hoc calibration with temperature scaling \citep{guo_onconfcal_2017}.

For a confidence estimate $c \in (0,1)$, binary temperature scaling
applies a scalar temperature $T>0$ in logit space,
\begin{equation}
    \widetilde{c}_{T}
    =
    \sigma\!\left(
        \frac{\operatorname{logit}(c)}{T}
    \right),
    \label{eq:temperature-scaling}
\end{equation}
where $\sigma$ denotes the logistic sigmoid. Following \citet{guo_onconfcal_2017}, we choose a temperature for each method by minimizing the negative-log likelihood (NLL) on the development set defined in \autoref{app:datasets}, choosing a maximum temperature of $T=15$ if the NLL has no finite minimum.

In \autoref{tab:temperature-scaling}, CRG remains the best-calibrated method after temperature scaling and achieves the highest BAS.

\begin{table*}[t]
\centering
\caption{Confidence estimation results on OpenHands trajectories, averaged over agent LLMs, after applying post-hoc temperature scaling. We tune the temperature parameter on the development set according to the procedure in \citet{guo_onconfcal_2017}. We report adaptive ECE (10 bins), Brier score, AUROC, and BAS. Best per column in \textbf{bold}, second \underline{underlined}. $^\dagger$White-box: uses token probabilities from a surrogate LLM.}
\label{tab:temperature-scaling}
\setlength{\tabcolsep}{3pt}
\resizebox{\textwidth}{!}{%
\begin{tabular}{ll cccc cccc cccc}
\toprule
& & \multicolumn{4}{c}{SWE-Bench Verified} & \multicolumn{4}{c}{EnterpriseOps-Gym} & \multicolumn{4}{c}{SkillsBench} \\
\cmidrule(lr){3-6} \cmidrule(lr){7-10} \cmidrule(lr){11-14}
Estimator & LLM & ECE$\downarrow$ & Brier$\downarrow$ & AUROC$\uparrow$ & BAS$\uparrow$ & ECE$\downarrow$ & Brier$\downarrow$ & AUROC$\uparrow$ & BAS$\uparrow$ & ECE$\downarrow$ & Brier$\downarrow$ & AUROC$\uparrow$ & BAS$\uparrow$ \\
\midrule
Basic Verbalizer + Temp. Scaling & Qwen-3.8 27B & 0.19 & \underline{0.22} & \textbf{0.66} & 0.34 & 0.16 & 0.26 & 0.60 & 0.07 & 0.14 & \underline{0.25} & \textbf{0.65} & \underline{0.15} \\
Reason-as-Graph + Temp. Scaling & Qwen-3.8 27B & \underline{0.16} & \textbf{0.21} & \textbf{0.66} & \underline{0.35} & 0.18 & 0.26 & \textbf{0.62} & 0.06 & 0.14 & \underline{0.25} & \underline{0.63} & $-$0.01 \\
Surrogate LNSP$^\dagger$ + Temp. Scaling & Qwen-3.8 27B & 0.22 & 0.24 & 0.47 & 0.32 & \underline{0.12} & \underline{0.25} & \underline{0.61} & \underline{0.08} & \underline{0.10} & \underline{0.25} & 0.46 & 0.14 \\
\midrule
\rowcolor{blue!8} \textbf{CRG (ours)} + Temp. Scaling & Qwen-3.8 27B & \textbf{0.13} & \textbf{0.21} & \underline{0.57} & \textbf{0.36} & \textbf{0.07} & \textbf{0.23} & \textbf{0.62} & \textbf{0.10} & \textbf{0.06} & \textbf{0.23} & \textbf{0.65} & \textbf{0.16} \\
\bottomrule
\end{tabular}
}
\end{table*}

\subsection{Graph Quality Analysis}
\label{sec:analysis-graph-quality}

Our method assumes a decomposition structure that an LLM may not produce: each decomposition should yield conditionally independent children whose conjunction is equivalent to their parent, and each particularization should preserve its parent's meaning.
To evaluate whether this assumption holds, we test the implied entailments, following evaluations of claim decompositions \citep{weir-etal-2024-enhancing, huang2025decmetricsstructuredclaimdecomposition} and of assurance case structure \citep{ikhwantri_explainable_2025}.
For a parent claim $G$ with children $G_1,\ldots,G_m$, and a particularized restatement $G'$ of $G$, we check four criteria:
\begin{itemize}
    \item \textbf{Sufficiency}: the conjunction of the children covers their parent, $\bigwedge_{j=1}^{m}G_j \rightarrow G$.
    \item \textbf{Necessity}: each child is required, $G \rightarrow G_j$ for every $j$.
    \item \textbf{Particularization}: a restatement and its original claim entail one another in context, $G' \rightarrow G$ and $G \rightarrow G'$.
    \item \textbf{Non-redundancy}: no sibling entails another, $G_j \not\models G_k$ and $G_k \not\models G_j$ for every $j \neq k$.
\end{itemize}

It can be shown from the implications above that, should all sufficiency, necessity, and particularization entailments hold, the leaves form the conjunction required by \autoref{eq:root-equiv}.
Non-redundancy serves as a diagnostic for violations of the conditional-independence assumption, but does not establish independence.
We evaluate 50 sampled graphs per benchmark and \methodnameshort{} setting (300 total).
Within each graph, we check all sufficiency, necessity, and particularization entailments, and sample at most 20 sibling pairs for non-redundancy to bound evaluation cost, since a parent with $m$ children induces $\binom{m}{2}$ pairwise comparisons.
A particularization holds only if both entailments hold, and a sibling pair is non-redundant only if neither claim entails or contradicts the other.
We score each graph as the fraction of checks that hold for each criterion and macro-average across graphs, so that a graph's contribution does not depend on its size.
We use the same sampled graphs and checks with three LLM entailment judges: our two CRG setting models (GPT-5.6 Sol, Qwen-3.8 27B), and Gemini-3.6 Flash \citep{google2026gemini36flash}.
\autoref{tab:graph-quality} shows the \methodnameshort{}s satisfy sufficiency and non-redundancy more consistently than necessity and particularization, suggesting that precise claim decomposition and contextualization remains difficult for LLMs.
Judges unanimously find \methodnameshort{}s with GPT 5.6 Sol to better maintain the refinement invariant our method assumes (\autoref{eq:refinement-invariant}), yet also unanimously find GPT 5.6 graphs to contain more overlapping sibling claims.
Thus, neither \methodnameshort{} setting achieves near perfect decomposition.
In \autoref{app:analysis-case-studies}, we evaluate case studies, including an example with a clear decomposition failure.
Nevertheless, these imperfect graphs produce the calibration and decision-utility gains in \autoref{tab:main-results}, indicating that perfect decomposition is not required for useful confidence estimates.

\begin{table*}[t]
\centering
\caption{Macro-averaged graph-quality scores by benchmark, \methodnameshort{} LLM, and judge.}
\label{tab:graph-quality}
\setlength{\tabcolsep}{2.5pt}
\resizebox{\textwidth}{!}{%
\begin{tabular}{ll cccc cccc cccc}
\toprule
& & \multicolumn{4}{c}{SWE-Bench Verified} & \multicolumn{4}{c}{EnterpriseOps-Gym} & \multicolumn{4}{c}{SkillsBench} \\
\cmidrule(lr){3-6} \cmidrule(lr){7-10} \cmidrule(lr){11-14}
CRG LLM & Judge
& Suff. & Nec. & Partic. & Non-red.
& Suff. & Nec. & Partic. & Non-red.
& Suff. & Nec. & Partic. & Non-red. \\
\midrule
\multirow{3}{*}{GPT-5.6 Sol} & GPT-5.6 Sol
& \textbf{0.98} & \textbf{0.79} & \textbf{0.75} & 0.82
& \textbf{0.96} & \textbf{0.77} & \textbf{0.65} & 0.95
& \textbf{0.90} & \textbf{0.86} & \textbf{0.69} & \textbf{0.84} \\
 & Qwen-3.8 27B
& \textbf{0.92} & \textbf{0.71} & \textbf{0.41} & 0.88
& \textbf{0.85} & \textbf{0.65} & \textbf{0.37} & 0.96
& \textbf{0.85} & \textbf{0.71} & \textbf{0.53} & \textbf{0.86} \\
 & Gemini-3.6 Flash
& \textbf{0.99} & \textbf{0.80} & \textbf{0.69} & 0.81
& \textbf{0.98} & \textbf{0.80} & \textbf{0.74} & 0.93
& \textbf{0.92} & \textbf{0.87} & \textbf{0.75} & \textbf{0.81} \\
\midrule
\multirow{3}{*}{Qwen-3.8 27B} & GPT-5.6 Sol
& 0.79 & 0.60 & 0.51 & \textbf{0.87}
& 0.92 & 0.59 & 0.30 & \textbf{0.96}
& 0.75 & 0.80 & 0.56 & 0.82 \\
 & Qwen-3.8 27B
& 0.67 & 0.38 & 0.29 & \textbf{0.91}
& 0.82 & 0.42 & 0.14 & \textbf{0.97}
& 0.68 & 0.61 & 0.38 & 0.81 \\
 & Gemini-3.6 Flash
& 0.82 & 0.60 & 0.54 & \textbf{0.83}
& 0.95 & 0.64 & 0.48 & \textbf{0.96}
& 0.79 & 0.77 & 0.59 & 0.77 \\
\bottomrule
\end{tabular}
}
\end{table*}

\subsection{Prompts}
\label{app:prompts}

We reproduce the prompts used for CRG construction, evidence gathering,
leaf-confidence estimation, and the prompted baselines. Angle-bracketed text
denotes content inserted at runtime. Task statements and trajectories are
represented by placeholders rather than repeated in full. %

\subsubsection{Domain-Specific Success Definitions}

To support estimating confidence in agent success, we found it important to precisely define what success \textit{means} for each task domain.
We think such a definition could be useful for any system which evaluate success through LLM reasoning:
it enables an estimate to more precisely cover a user's desiderata for a task, rather than relying on potentially incorrect assumptions an LLM might make about a domain.
For example, when using a coding agent, one might or might not want to include writing ``clean, well-documented code'' as a criterion for success, depending on the context.
In either case, the confidence estimate for success should reflect exactly what it means to be successful in that domain or setting.
The following domain-specific definitions are inserted into \textbf{both CRG construction \textit{and} the verbalized baseline prompts.}

\paragraph{Software Engineering (SWE-Smith \& SWE-Bench).}

\begin{Prompt}
## The Agent's Success Criteria

A SWE agent has successfully accomplished its user's task if and only if both
of the following hold:

1. It has correctly implemented every requirement that the issue description
   states or implies.
2. It has introduced no regressions into the repository.
\end{Prompt}

\paragraph{Enterprise Ops.}
\begin{Prompt}
## The Agent's Success Criteria

An enterprise agent has successfully accomplished its user's task if and only
if all of the following hold:

1. The user's primary task objective was fully achieved.
2. The integrity constraints of the system, such as foreign key
   relationships, are respected and maintained.
3. The agent complied with system policies governing permissions and
   procedural rules.
\end{Prompt}

\paragraph{Skills Bench.}
\begin{Prompt}
## The Agent's Success Criteria

An agent has successfully accomplished its user's task if and only if all of
the following hold:

1. All the requirements in the user's stated task have been achieved.
2. All required artifacts have been produced correctly.
\end{Prompt}

\subsubsection{CRG Construction}
\label{app:crg_construction}

Graph construction uses one system message followed by a user message for
claim decomposition. Evidence gathering is then requested in the same
conversation, so the constructor retains the trajectory, the constructed
claim tree, and the preceding tool results.

\paragraph{Constructor system message.}
\begin{Prompt}
You are a confidence-grounded assurance case graph constructor. Your task is
to construct a structured, graph-based argument which decomposes and evaluates
the claim that an agent successfully accomplished its task, based on a
completed agent trajectory. The graph is used to estimate confidence that the
agent's final delivered result substantively satisfies the user's task.

## What is a Confidence-Grounded Assurance Graph?

A confidence-grounded assurance graph is a structured argument that the
agent's final result satisfies the overall claim -- built by decomposing that
claim into a tree of necessary-and-sufficient success conditions until each
can be directly evaluated with trajectory evidence. The graph is rooted in a
top-level claim that the agent's final result satisfies the user's task. This
claim is recursively developed into sub-claims, until each leaf claim is
atomic and can be directly evaluated with trajectory evidence.

Thus, the graph has two node types: goal nodes, which contain an auditable
claim or sub-claim in the argument tree, and evidence nodes, which cite
concrete, specific, trajectory-grounded observations that evaluate a claim.

### Goal Nodes

The root claim and every sub-goal derived from it.

- Each claim is a success condition: a falsifiable criterion that the agent's
  final result must satisfy for its parent claim to hold.
- Claims may be abstract (task-agnostic) or concrete (task-aware).
- Use divide-and-conquer to decompose a claim into children at approximately
  the same level of concreteness.
- For every parent claim, its children must be individually necessary and
  collectively sufficient. The parent is true when all of its children are
  true, and if any child is false, the parent is false.
- Use particularization to restate one abstract claim as one concrete claim
  using details from the specific task. The particularized claim must neither
  weaken nor strengthen its parent; in the task context, both must have the
  same truth condition.
- Describe what must be true of the final result, not how the agent tried to
  achieve it. Different correct implementations or solutions should satisfy
  the same claims.
- Do not use optional actions, implementation choices, or verification steps
  as success conditions -- they may instead be evidence.
- An atomic claim is a single, independently verifiable claim with enough
  context to assess directly against evidence without referring to its parent.
- Stop when every leaf is atomic. Split claims requiring separate evidence;
  do not split attributes, examples, or checks that can be evaluated together
  from the same evidence.

### Evidence Nodes

Concrete, trajectory-grounded observations that bear on one or more claims
through a permitted evidence relationship.

- A claim may need several pieces of evidence to be fully assessed; reuse the
  same evidence node across multiple claims when one observation is relevant
  to more than one.
- Missing evidence creates uncertainty, not falsity. You can use evidence
  notes to state the absence of a particular action or observation in the
  trajectory which you expect to need to fully evaluate a claim.

### Edges

Edges are labeled with one of the following relationship types:

- supports: increases confidence in the target claim
- undermines: decreases confidence in the target claim
- unverified: identifies missing verification or evidence, without indicating
  falsehood
- decomposes_from: is a conjunct in a decomposition of the target goal
- particularizes: restates the target goal concretely for the specific task
  context

## Process

You will be given an overall goal or claim to evaluate against the trajectory.
Build the graph in two phases:

1. Develop claims -- iteratively use divide-and-conquer and particularization
   to decompose the overall claim down to atomic claims, forming the claim tree.
2. Gather evidence -- add evidence nodes to the graph to cite trajectory
   evidence that relates meaningfully to each atomic claim, or note where
   required evidence is missing.

A separate confidence estimator will use the finished graph to assign
confidence values to each claim -- your job is the graph construction only.

<DOMAIN-SPECIFIC SUCCESS CRITERIA>

## Your Success Criteria

The finished graph should let a downstream evaluator assign confidence to the
overall claim using only what is in the graph:

- The claim tree begins with an abstract decomposition of the overall claim,
  considering sub-claims that are relevant to the whole domain, not this
  individual user's task or environment.
- Every divide-and-conquer parent is decomposed into child success conditions
  that are individually necessary and collectively sufficient to evaluate the
  parent claim.
- Every particularized child is a concrete, task-aware restatement with the
  same truth condition as its abstract parent.
- Every atomic claim describes a final result or success condition, not the
  observed path taken to reach it.
- Every atomic claim is evaluated with evidence, or honestly marked uncertain
  where the trajectory does not address it.
- Nothing relevant is missing; nothing irrelevant is included.
\end{Prompt}

\paragraph{Claim-decomposition user message.}
The maximum root-to-leaf depth is instantiated as five in all reported runs.

\begin{Prompt}
<trajectory>
<AGENT TRAJECTORY>
</trajectory>

<users_task>
<USER TASK>
</users_task>

Construct the success-condition claim tree for:

- Overall Goal Identifier: Successful Final Outcome
- Overall Claim: The agent's final delivered result satisfies the substantive
  success criteria of the user's task

## Phase 1: Decompose and Particularize Claims

Iteratively decompose and particularize the Overall Claim until every leaf is
an atomic, independently verifiable claim. The goal tree may have at most 5
strategy edges on any root-to-leaf path. Use the following tools:

1. `divide_and_conquer` decomposes an abstract or concrete claim into two or
   more claims at approximately the same level of concreteness. For each
   decomposition:
   - Every child must be necessary for its parent.
   - All children together must be sufficient for their parent.
   - The parent must be true exactly when all its children are true.
2. `particularize` restates one abstract, task-agnostic claim as one concrete,
   task-aware claim. The new claim must neither weaken nor strengthen its
   parent; in the task context, both must have the same truth condition.
   Generally, an abstract claim should be particularized at most once.

### Guidance

- Divide broad abstract claims into smaller, still abstract claims before
  particularizing them. Particularize only when the claim is concrete enough
  to be evaluated in the context of the user's task.
- Do not jump from an abstract claim to concrete task details through
  divide-and-conquer. Maintain a consistent level of abstraction between
  parent and child claims. Use particularization to move from abstract to
  concrete.
- Apply the system prompt's atomicity rule: split claims requiring separate
  evidence, but keep attributes, examples, or checks together when the same
  evidence evaluates them.
- Keep enough context that each claim is clear and evaluable without reading
  its parent.
- Describe what must be true of the final result, not how the agent tried to
  achieve it. Different correct implementations or solutions should satisfy
  the same claims.
- Prefer strong, clear success-condition claims. Do not weaken or strengthen a
  claim when particularizing. Do not weaken claims to match the observed
  trajectory or turn actions, tests, or implementation details into success
  conditions.
- Use unique goal identifiers and each target goal's exact identifier and
  claim.
- Make as many independent strategy calls in each step as you can determine
  soundly, as long as they develop different claims within the graph. After
  observing the updated graph, continue developing any non-atomic leaves.
  When the complete graph has only atomic leaves, call `finish`.
- In general: prefer simple graphs.

Produce claim structure only in this phase. Do not cite or gather evidence yet.
\end{Prompt}

The constructor is given the following tool interfaces. Each generated
sub-goal contains a unique identifier, an auditable claim, and a rationale for
its inclusion.

\begin{Prompt}
divide_and_conquer(
    target_goal_identifier,
    target_auditable_claim,
    sub_goals=[
        {
            goal_identifier,
            auditable_claim,
            reasoning
        },
        ...
    ]
)

particularize(
    target_goal_identifier,
    target_auditable_claim,
    particularized_goal={
        goal_identifier,
        auditable_claim,
        reasoning
    }
)

finish(message)
\end{Prompt}

\paragraph{Evidence-gathering user message.}
This message is appended to the graph-construction conversation after the
constructor finishes the claim tree. Evidence is produced in one structured
tool call.

\begin{Prompt}
## Phase 2: Gather Evidence

For every atomic leaf claim in the preceding decomposition, gather the
concrete trajectory observations needed to assess it. Checking one claim may
use several pieces of evidence, and one piece of evidence may be relevant to
multiple claims.

- For each claim, include both supporting and undermining evidence if present.
- Reuse an evidence item when the same observation contributes to multiple
  claims.
- Cite the precise trajectory step numbers for every observed action or event.
- Do not infer success merely because no failure was observed.

If the trajectory lacks relevant evidence for a claim, create an evidence item
stating that evidence is absent, use an empty `step_numbers` list. Choose an
appropriate edge label. The absence may leave the claim uncertain; it does not
necessarily establish that the claim is false.

You have exactly 1 step to call `gather_evidence_for_agentic_graph_v1` and
record the evidence catalog and its typed edges.
\end{Prompt}

The evidence tool returns a reusable catalog and a set of typed attachments:

\begin{Prompt}
gather_evidence_for_agentic_graph_v1(
    evidence_catalog=[
        {
            evidence_key,
            evidence,
            step_numbers,
            auditable_claim,
            contribution
        },
        ...
    ],
    evidence_edges=[
        {
            evidence_key,
            target_goal_identifier,
            relationship_type:
                supports | undermines | unverified
        },
        ...
    ]
)
\end{Prompt}

\subsubsection{CRG Leaf-Confidence Estimation}

Each leaf goal is scored independently in a new LLM call. The trajectory is
reduced as described in \autoref{sec:methods-verbalizing}; evidence attached
to the target leaf is supplied in the related-claims block. Evidence nodes
themselves are not assigned separate confidence values in the reported
\texttt{goal\_leaves\_product} configuration.

\paragraph{Estimator system message.}
\begin{Prompt}
You estimate calibrated confidence for claims in a confidence-grounded
assurance graph. Such a graph is a structured argument that an agent
successfully completed its task: goal nodes decompose that overall claim into
necessary success conditions, and evidence nodes connect those claims to
concrete trajectory observations. Evaluate only the target claim using the
supplied trajectory and related claims. Return the requested structured
confidence and a brief rationale.

## Your Success Criteria

- You provide a clear, calibrated confidence that reflects both the available
  evidence for and against a claim, and the degree of uncertainty that remains
  in the absence of further evidence.
- You provide a clear rationale that matches your confidence score.
- Your confidence score is calibrated, such that when normalized to a
  probability, the score reflects the actual likelihood that the claim is
  true.
\end{Prompt}

\paragraph{Leaf-scoring user message.}
The related-claims entry is repeated for every evidence item attached to the
target leaf.

\begin{Prompt}
# Context

<overall_goal>
Goal: Successful Final Outcome
Claim: The agent's final delivered result satisfies the substantive success
criteria of the user's task
</overall_goal>

<problem_statement>
<USER TASK>
</problem_statement>

<trajectory>
<REDUCED AGENT TRAJECTORY>
</trajectory>

<related_claims>
1. Relationship to target: <supports | undermines | unverified>
   Related Claim: <EVIDENCE CLAIM>
   Evidence Steps: <COMMA-SEPARATED STEP NUMBERS>
   Contribution: <EVIDENCE CONTRIBUTION>

<ADDITIONAL ATTACHED EVIDENCE, IF ANY>
</related_claims>

# Task Instructions

Estimate whether the target claim is true and can soundly support its parent
claim. Judge this claim alone, not whether the overall task succeeded. Weigh
predecessor confidences as evidence rather than averaging them. In the absence
of evidence, consider the claim, any available trajectory information, and
your priors about this problem.

Think carefully about evidence relationships:

- if evidence suggests a claim is unverified, do not assume it is False.
  Instead, establish confidence considering your priors: how likely is it that
  this claim is true, even though it was not formally verified by the agent?
  Such claims may depend on world knowledge, or reasoning inferences about the
  problem which were not formally verified, but may nonetheless be true. Use
  confidence to reflect uncertainty, not certainty in falsehood. For
  reference, a uniform prior suggests a confidence of 50%

Confidence guidance:

- 80--100%
  the claim.
- 60--80%
  minor.
- 40--60%
- 20--40%
  false than true.
- 0--20%
  the claim, establishing its falsehood.

Return a confidence from 0%
a brief rationale.

If you are using a percentage scale, please round your confidence to the
nearest whole number between 0 and 100.

<target_goal>
Goal: <LEAF GOAL IDENTIFIER>
Claim: <LEAF CLAIM>
Reasoning: <RATIONALE FOR INCLUDING THE LEAF>
</target_goal>
\end{Prompt}

The response is constrained to the structured schema
\texttt{\{confidence: float, rationale: string\}}. The implementation rejects
scores outside $[0,100]$ and non-integer confidence values, then divides the
accepted value by $100$ before aggregation.

\subsubsection{Prompted Baselines}

All black-box baselines receive the following common prompt prefix. Their
method-specific instructions, given below, are appended after
\texttt{\# Task Instructions}.

\begin{Prompt}
You are evaluating whether an agent successfully completed a user's task.
Review the complete agent trajectory below, including its final answer and any
tool results.

# Context

<domain_success_criteria>
<DOMAIN-SPECIFIC SUCCESS CRITERIA>
</domain_success_criteria>

<problem_statement>
<USER TASK>
</problem_statement>

<agent_trajectory>
<COMPLETE AGENT TRAJECTORY>
</agent_trajectory>

# Task Instructions
<METHOD-SPECIFIC INSTRUCTIONS>
\end{Prompt}

\paragraph{Basic Verbalizer.}
The method-specific instruction is:

\begin{Prompt}
Reason about the correctness of the completed work, then end your response
with a confidence percentage on its own line in exactly this form:
"Confidence: N%
\end{Prompt}

\paragraph{Reason-as-Graph.}
This baseline uses the Basic Verbalizer output instruction followed by the
additional graph-oriented guidance below. It does not construct or return a
persistent graph.

\begin{Prompt}
Reason about the correctness of the completed work, then end your response
with a confidence percentage on its own line in exactly this form:
"Confidence: N%

## Guidance

We suggest using a confidence-grounded assurance case graph to structure your
reasoning.

A confidence-grounded assurance case graph is a structured argument for
whether a task was completed correctly. It starts with that overall claim:
"the submitted result is correct", breaks it into a complete set of smaller,
auditable claims, and evaluates each claim using concrete evidence from the
trajectory. Evidence may support or undermine a claim, or reveal uncertainty
where the trajectory is incomplete. Together, the claims and evidence should
account for every material reason the task may have succeeded or failed.

Use this structure as a scaffold for your confidence reasoning. Break
correctness into complete, checkable parts; assess each part against relevant
trajectory evidence; then combine those assessments into an overall
confidence. You do not need to output a graph.
\end{Prompt}

\paragraph{Verbal Consistency.}
For each trajectory, we independently sample ten responses to the common
prompt with the following instruction. Confidence is the fraction of valid
responses ending in \texttt{Successful: True}.

\begin{Prompt}
Reason about the correctness of the completed work, then end your response
with your answer on its own line in exactly this form:
"Successful: <answer>", where <answer> is either True or False.
\end{Prompt}

\paragraph{Surrogate LNSP.}
Surrogate LNSP does not use a natural-language confidence-estimation prompt.
It computes the length-normalized sequence probability of the agent's final
action under the surrogate LLM using the reconstructed agent context.

\subsubsection{Reason-with-Graph Ablation}

The Reason-with-Graph ablation in
\autoref{sec:analysis-component-ablation} receives the already constructed
CRG rendered as a nested textual assurance case. It directly verbalizes
confidence in the root rather than scoring and aggregating individual leaves.

\begin{Prompt}
You are evaluating whether an agent successfully completed a user's task
using a confidence-grounded assurance case.

The assurance case is a graph rooted at an overall goal. Goal nodes state
auditable claims and may be decomposed into subgoals. Evidence nodes contain
concrete observations. Each connection states whether its source decomposes
the target goal, or provides some meaningful evidence which supports,
undermines, or identifies uncertainty in the target claim. Judge the root
claim using only this graph, accounting for the completeness of its
decomposition, the strength and relevance of its evidence, missing support,
and conflicting evidence. Do not assume that every claim or item of evidence
is correct merely because it appears in the graph.

<assurance_case>
<RENDERED CONFIDENCE REASONING GRAPH>
</assurance_case>

Reason about the correctness of the completed work, then end your response
with a confidence percentage on its own line in exactly this form:
"Confidence: N%
\end{Prompt}

\subsection{Behavioral Alignment Score}
\label{app:bas}
\citet{wu2026bas} evaluate confidence by the quality of the answer-or-abstain
decisions it supports.
A decision-maker with risk tolerance $t \in [0,1)$ receives utility $1$ for
accepting a correct solution, $-\tfrac{t}{1-t}$ for accepting an incorrect one,
and $0$ for abstaining.
Given a confidence $\agentconf$, accepting has expected utility
$\agentconf - (1-\agentconf)\tfrac{t}{1-t}$, which is non-negative exactly when
$\agentconf \geq t$; the threshold $t$ is therefore the confidence level at
which a user with that risk tolerance becomes willing to accept.
BAS aggregates the realized utility of the induced policy over a uniform prior
on $t$.
Because the policy abstains for $t > \agentconf$, the integral runs only over
$[0, \agentconf]$ and admits a closed form:
\begin{equation}
    U(\agentconf, \agentsuccess) =
    \begin{cases}
        \agentconf, & \agentsuccess = 1, \\[2pt]
        \agentconf + \ln(1 - \agentconf), & \agentsuccess = 0,
    \end{cases}
    \qquad
    \mathrm{BAS} = \frac{1}{n} \sum_{i=1}^{n} U(\agentconf_i, \agentsuccess_i).
    \label{eq:bas}
\end{equation}
A correct solution thus contributes utility equal to the confidence assigned to
it, while an incorrect one is penalized logarithmically, so
$\mathrm{BAS} \in (-\infty, 1]$.
\citet{wu2026bas} further show that BAS is a proper decision-theoretic scoring
objective: expected utility is uniquely maximized when reported confidence
equals the true probability of correctness, so BAS rewards calibration rather
than ranking alone.

\paragraph{Overconfident errors.}
The term $\ln(1-\agentconf)$ diverges as $\agentconf \to 1$, so a single
incorrect trajectory assigned near-certain confidence can dominate the score.
Following \citet{wu2026bas} we clip confidences to
$\min(\agentconf, 1-\epsilon)$ with $\epsilon = 10^{-4}$, which bounds the
worst per-trajectory contribution at $\approx -8.21$.
This asymmetry is the mechanism behind the large negative BAS values we report
for verbalized and sampling-based baselines (\autoref{tab:main-results}): both
place near-maximal confidence on trajectories that failed, which ECE penalizes
only in proportion to the gap while BAS penalizes without bound.

\subsection{Theoretical Foundations of CRG Confidence Aggregation}
\label{app:detailed_proof}

Here we formalize the probabilistic justification for our
aggregation rule. We proceed in three steps. First, we show that if
each refinement preserves the intended conjunctive semantics, then the
root claim is equivalent to the conjunction of all leaf claims.
Second, we characterize the root probability without making any
independence assumption. Finally, we recover the product rule used by
our estimator as the special case in which leaf claims are conditionally
independent. 

Throughout, equivalence $\equiv$ is understood in the task context:
$G \equiv G'$ means that the two claims have the same truth value for
every execution under consideration. 
We identify each claim $G$ with the event that $G$ is true. Probabilities are epistemic and conditioned on the information available to the estimator from the observed trajectory $\tau$; the trajectory need not fully determine the claim's truth.

\subsubsection{Leaf Equivalence}
\label{app:proofs-leaf-equiv}

Recall the refinement invariant of
\autoref{sec:methods-graph-claim-decomp}: for every internal claim
$G \in \claimset \setminus \leafset$, the conjunction of its children
is intended to be equivalent to the parent,
\begin{equation}
    G \;\equiv\;
    \bigwedge_{G' \in \childfn(G)} G' .
    \tag{\ref{eq:refinement-invariant}}
\end{equation}
Decomposition realizes this with $|\childfn(G)|>1$, while
particularization realizes it with $|\childfn(G)|=1$.

For $G\in\claimset$, let $\leafset(G)$ denote the leaves of the
refinement subtree rooted at $G$, so that
$\leafset(G_0)=\leafset$.

\begin{lemma}[Leaf equivalence]
\label{lemma:leaf_equivalence}
Let $\crg$ be a \methodnameshort{} whose refinement structure satisfies
\autoref{eq:refinement-invariant} at every internal claim. Then, for
every $G\in\claimset$,
\begin{equation}
    G
    \;\equiv\;
    \bigwedge_{H\in\leafset(G)} H .
    \label{eq:subtree-equiv}
\end{equation}
In particular,
\begin{equation}
    G_0
    \;\equiv\;
    \bigwedge_{H\in\leafset} H .
    \label{eq:root-leaf-equiv}
\end{equation}
\end{lemma}

\begin{proof}
We proceed by structural induction on the subtree rooted at $G$.
The induction is well-founded because
\autoref{sec:methods-graph-claim-decomp} performs at most $k$
refinement rounds, so the refinement structure is a finite tree.

\emph{Base case.}
If $G\in\leafset$, then $\leafset(G)=\{G\}$, and
\autoref{eq:subtree-equiv} reduces to $G\equiv G$.

\emph{Inductive step.}
Let $G$ be an internal claim with
$\childfn(G)=\{G_1,\ldots,G_m\}$, and suppose
\autoref{eq:subtree-equiv} holds for each child $G_j$.
By the refinement invariant,
\[
    G \equiv \bigwedge_{j=1}^{m} G_j ,
\]
and by the inductive hypothesis,
\[
    G_j
    \equiv
    \bigwedge_{H\in\leafset(G_j)} H .
\]
Substituting these equivalent claims gives
\begin{align}
    G
    &\equiv
    \bigwedge_{j=1}^{m}
    \bigwedge_{H\in\leafset(G_j)} H \\
    &=
    \bigwedge_{H\in\leafset(G)} H ,
\end{align}
because the leaves below $G$ are precisely the leaves below its
children. Taking $G=G_0$ yields
\autoref{eq:root-leaf-equiv}.
\end{proof}

\begin{remark}[Logical overlap]
\label{rem:overlap}
Lemma~\ref{lemma:leaf_equivalence} does not require distinct leaf
claims to be logically independent. Distinct leaves may express
overlapping, correlated, or even equivalent conditions without
invalidating the logical equivalence in
\autoref{eq:root-leaf-equiv}; conjunction remains well-defined in
each case. Such dependence matters only when we attempt to recover
the probability of the conjunction from individual leaf
probabilities.
\end{remark}

\begin{remark}[Status of the refinement assumption]
\label{rem:intended-semantics}
The refinement invariant specifies the semantics that our construction
is intended to realize, rather than a property guaranteed by the
construction. Because the refinements are generated by an LLM, their
children need not always form a conjunction exactly equivalent to the
parent. Lemma~\ref{lemma:leaf_equivalence} therefore characterizes
graphs satisfying the intended semantics;
\autoref{sec:analysis-graph-quality} empirically measures departures
from this assumption.
\end{remark}

\subsubsection{Dependence-Aware Aggregation}
\label{app:proofs-dependence}

Lemma~\ref{lemma:leaf_equivalence} reduces confidence in the root
claim to confidence in a conjunction of leaf claims. Importantly,
this reduction itself requires no probabilistic independence
assumption. We first characterize the corresponding probability under
arbitrary dependence.

\begin{proposition}[Dependence-aware leaf aggregation]
\label{prop:dependent_leaf_aggregation}
Suppose $\crg$ satisfies
\autoref{eq:refinement-invariant} at every internal claim, and let
$\leafset=\{H_1,\ldots,H_m\}$. Then, for any ordering of the leaves
for which the relevant conditional probabilities are defined,
\begin{equation}
    P(G_0\mid\agenttraj)
    =
    \prod_{i=1}^{m}
    P\!\left(
        H_i
        \,\middle|\,
        \agenttraj,H_1,\ldots,H_{i-1}
    \right),
    \label{eq:dependent-chain-rule}
\end{equation}
where the first factor is $P(H_1\mid\agenttraj)$.

If only the marginal leaf probabilities
\[
    p_i=P(H_i\mid\agenttraj)
\]
are known, then without any assumption on their dependence,
\begin{equation}
    \max\!\left\{
        0,\,
        \sum_{i=1}^{m}p_i-(m-1)
    \right\}
    \;\leq\;
    P(G_0\mid\agenttraj)
    \;\leq\;
    \min_i p_i .
    \label{eq:leaf-probability-bounds}
\end{equation}
\end{proposition}

\begin{proof}
By Lemma~\ref{lemma:leaf_equivalence},
\[
    G_0
    \equiv
    \bigwedge_{i=1}^{m} H_i ,
\]
so
\begin{equation}
    P(G_0\mid\agenttraj)
    =
    P\!\left(
        \bigcap_{i=1}^{m}H_i
        \,\middle|\,
        \agenttraj
    \right).
\end{equation}
Applying the probability chain rule gives
\begin{align}
    P(G_0\mid\agenttraj)
    &=
    P(H_1\mid\agenttraj)
    \prod_{i=2}^{m}
    P\!\left(
        H_i
        \,\middle|\,
        \agenttraj,H_1,\ldots,H_{i-1}
    \right),
\end{align}
establishing \autoref{eq:dependent-chain-rule} without an independence
assumption.

For the upper bound, an intersection cannot be more probable than any
of its constituent events:
\begin{equation}
    P(G_0\mid\agenttraj)
    \leq
    \min_i P(H_i\mid\agenttraj)
    =
    \min_i p_i .
\end{equation}

For the lower bound, the union bound applied to the complementary
events gives
\begin{align}
    P\!\left(
        \bigcup_i \neg H_i
        \,\middle|\,
        \agenttraj
    \right)
    &\leq
    \sum_i P(\neg H_i\mid\agenttraj) \\
    &=
    \sum_i(1-p_i).
\end{align}
Therefore,
\begin{align}
    P(G_0\mid\agenttraj)
    &=
    1-
    P\!\left(
        \bigcup_i \neg H_i
        \,\middle|\,
        \agenttraj
    \right) \\
    &\geq
    1-\sum_i(1-p_i) \\
    &=
    \sum_i p_i-(m-1).
\end{align}
Combining this with non-negativity yields the lower bound in
\autoref{eq:leaf-probability-bounds}.
\end{proof}

Proposition~\ref{prop:dependent_leaf_aggregation} makes explicit what
can and cannot be inferred from leaf-level confidence alone. Under the refinement invariant, the root probability is the joint probability that all leaf claims hold, but the marginal probabilities $P(H_i\mid\agenttraj)$ generally do
not determine this joint probability. Without additional information
about dependence, they determine only the interval in
\autoref{eq:leaf-probability-bounds}.

\subsubsection{Product Aggregation}
\label{app:proofs-product}

The product rule used by our method follows immediately when the
dependence structure is approximated by conditional independence.

\begin{corollary}[Product rule under local conditional independence]
Suppose the refinement invariant above holds at every internal claim, and
suppose that for every internal claim $G \in V \setminus L$, the claims in
$\mathrm{ch}(G)$ are mutually independent conditional on $\tau$. 
Then
\[
P(G_0 \mid \tau)
=
\prod_{H \in L} P(H \mid \tau).
\]
\end{corollary}

\begin{proof}
We proceed by structural induction. For a leaf $G$, the result is immediate.
For an internal claim $G$ with children $\mathrm{ch}(G)$, the refinement
invariant and conditional independence give
\[
P(G \mid \tau)
=
\prod_{G' \in \mathrm{ch}(G)} P(G' \mid \tau).
\]
Applying the induction hypothesis to each child and noting that their
subtree leaves together form $L(G)$ gives
\[
P(G \mid \tau)
=
\prod_{H \in L(G)} P(H \mid \tau).
\]
Taking $G=G_0$ yields the result.
\end{proof}

\paragraph{Implications for our estimator.}
Our estimator does not observe the true marginal probabilities.
Instead, it obtains a verbalized confidence
\begin{equation}
    c_H=\leafconf(H)
    \approx
    P(H\mid\agenttraj)
\end{equation}
for each leaf and uses the plug-in estimator
\begin{equation}
    \confestimator(\agenttraj)
    =
    \prod_{H\in\leafset} c_H .
    \label{eq:plugin-product}
\end{equation}
Thus, product aggregation should be understood as a tractable
conditional-independence approximation to the joint probability in
Proposition~\ref{prop:dependent_leaf_aggregation}, rather than as a
claim that the leaf conditions produced by the CRG are exactly
independent.

If the $c_H$ were exact marginal probabilities, but no dependence
assumption were made, the root probability would only be identified
within
\begin{equation}
    \max\!\left\{
        0,\,
        \sum_{H\in\leafset}c_H-(|\leafset|-1)
    \right\}
    \leq
    P(G_0\mid\agenttraj)
    \leq
    \min_{H\in\leafset}c_H .
    \label{eq:estimated-leaf-bounds}
\end{equation}
Because the $c_H$ are themselves estimates,
\autoref{eq:estimated-leaf-bounds} is only a plug-in characterization,
not a guaranteed confidence interval.

Dependence provides one explanation for systematic error in the
product approximation. When different leaves encode overlapping or
positively associated conditions, multiplying their marginal
confidences can count related uncertainty multiple times and
underestimate the joint probability. Other dependence structures can
instead cause overestimation. Our graph-quality analysis therefore uses non-redundancy among siblings as a diagnostic for obvious violations of the independence approximation. Non-redundancy does not establish independence, but substantial overlap indicates that product aggregation may double-count uncertainty. Non-redundancy does not establish independence, but
strong overlap is a direct warning sign that product aggregation may
double-count uncertainty.

A fully dependence-aware estimator could instead estimate the
sequential conditionals appearing in
\autoref{eq:dependent-chain-rule},
\begin{equation}
    \tilde c_i
    \approx
    P\!\left(
        H_i
        \,\middle|\,
        \agenttraj,H_1,\ldots,H_{i-1}
    \right),
\end{equation}
and return
\begin{equation}
    F_{\theta,\mathrm{dep}}(\agenttraj)
    =
    \prod_{i=1}^{m}\tilde c_i .
\end{equation}
With exact conditional probabilities, the chain-rule identity is
independent of the ordering of the leaves. In practice, however,
estimating these higher-order conditionals would require more complex
queries, and approximation error could make the resulting estimate
sensitive to the chosen ordering. We therefore use marginal leaf
scoring with product aggregation as a simple training-free
approximation and evaluate its empirical consequences through graph
quality and aggregation-rule ablations.

\begin{remark}[Bottom-up computation]
\label{rem:bottom-up}
Because multiplication is associative, the product in
\autoref{eq:plugin-product} can equivalently be computed bottom-up,
taking the value of each internal claim to be the product of its
children's values. The refinement-tree structure therefore determines
which leaf claims are evaluated and how they were derived from the
root, but not the numerical value of the final product once the leaf
confidences are fixed. This equivalence is specific to associative
aggregation; for example, recursively applying an arithmetic or
geometric mean can differ from applying the same operator once to all
leaves.
\end{remark}

\begin{remark}[Conservatism relative to leaf estimates]
\label{rem:monotone}
Since every $c_H\in[0,1]$,
\begin{equation}
    \confestimator(\agenttraj)
    \leq
    \min_{H\in\leafset} c_H .
\end{equation}
Thus, under product aggregation, one low-confidence leaf cannot be
offset by high-confidence values on other leaves, and introducing
additional leaves weakly decreases the estimate if existing
confidences are held fixed. This is an algebraic property of the
estimator; when conditional independence fails, it does not imply that
the product is necessarily a lower bound on the true root probability.
\end{remark}

\subsection{Paired Bootstrap Analysis}
\label{app:paired-bootstrap}

We check whether CRG's gains in Table~\ref{tab:main-results} persist under task resampling. For each benchmark--baseline--metric comparison, we create 10,000 bootstrap resamples by drawing the benchmark's original number of task IDs with replacement and retaining their available OpenHands trajectories. We compute the metric for Qwen CRG and the baseline on the same resample (paired bootstrap; \citealp{koehn-2004-statistical}). This uses saved predictions without new inference. \autoref{tab:paired-bootstrap} shows signed differences; asterisks mark nominal 95\% bootstrap intervals that exclude zero.

\begin{table*}[t]
\centering
\caption{Paired task-bootstrap comparison of CRG (Qwen-3.8 27B) with every baseline in \autoref{tab:main-results}. Entries are signed score differences: baseline minus CRG for ECE and Brier, and CRG minus baseline for AUROC and BAS. Positive values favor CRG. Exploratory mean rows average the six signed differences above them, with intervals recomputed on common task resamples. $^{*}$ marks a difference whose nominal 95\% bootstrap interval excludes zero, indicating statistical significance. No multiple-comparison adjustment is applied.}
\label{tab:paired-bootstrap}
\setlength{\tabcolsep}{5pt}
\resizebox{\textwidth}{!}{%
\begin{tabular}{llcccc}
\toprule
Benchmark & Baseline & ECE & Brier & AUROC & BAS \\
\midrule
SWE-Bench Verified & Basic Verbalizer (GPT) & $+0.115^{*}$ & $+0.034^{*}$ & $+0.006$ & $+0.411^{*}$ \\
 & Basic Verbalizer (Qwen) & $+0.095^{*}$ & $+0.018$ & $-0.084^{*}$ & $+0.101^{*}$ \\
 & Reason-as-Graph (GPT) & $+0.081^{*}$ & $+0.019$ & $-0.002$ & $+0.318^{*}$ \\
 & Reason-as-Graph (Qwen) & $+0.092^{*}$ & $+0.017$ & $-0.088^{*}$ & $+0.102^{*}$ \\
 & Verbal Consistency & $+0.154^{*}$ & $+0.052^{*}$ & $+0.030$ & $+7.959^{*}$ \\
 & Surrogate LNSP & $+0.009$ & $+0.004$ & $+0.098^{*}$ & $+0.011$ \\
\cmidrule(lr){2-6}
 & \textbf{Mean of six baselines} & $+0.091^{*}$ & $+0.024^{*}$ & $-0.007$ & $+1.484^{*}$ \\
\midrule
EnterpriseOps-Gym & Basic Verbalizer (GPT) & $+0.454^{*}$ & $+0.330^{*}$ & $+0.125^{*}$ & $+2.562^{*}$ \\
 & Basic Verbalizer (Qwen) & $+0.365^{*}$ & $+0.228^{*}$ & $+0.019$ & $+0.652^{*}$ \\
 & Reason-as-Graph (GPT) & $+0.442^{*}$ & $+0.320^{*}$ & $+0.150^{*}$ & $+2.771^{*}$ \\
 & Reason-as-Graph (Qwen) & $+0.331^{*}$ & $+0.199^{*}$ & $-0.002$ & $+0.577^{*}$ \\
 & Verbal Consistency & $+0.321^{*}$ & $+0.204^{*}$ & $+0.004$ & $+13.805^{*}$ \\
 & Surrogate LNSP & $+0.096^{*}$ & $+0.034^{*}$ & $+0.010$ & $+0.042^{*}$ \\
\cmidrule(lr){2-6}
 & \textbf{Mean of six baselines} & $+0.335^{*}$ & $+0.219^{*}$ & $+0.051^{*}$ & $+3.402^{*}$ \\
\midrule
SkillsBench & Basic Verbalizer (GPT) & $+0.344^{*}$ & $+0.217^{*}$ & $+0.110$ & $+2.312^{*}$ \\
 & Basic Verbalizer (Qwen) & $+0.299^{*}$ & $+0.166^{*}$ & $-0.002$ & $+0.474^{*}$ \\
 & Reason-as-Graph (GPT) & $+0.315^{*}$ & $+0.200^{*}$ & $+0.140^{*}$ & $+2.233^{*}$ \\
 & Reason-as-Graph (Qwen) & $+0.268^{*}$ & $+0.136^{*}$ & $+0.018$ & $+0.545^{*}$ \\
 & Verbal Consistency & $+0.332^{*}$ & $+0.197^{*}$ & $+0.043$ & $+14.012^{*}$ \\
 & Surrogate LNSP & $+0.051$ & $+0.042^{*}$ & $+0.192^{*}$ & $+0.059^{*}$ \\
\cmidrule(lr){2-6}
 & \textbf{Mean of six baselines} & $+0.268^{*}$ & $+0.160^{*}$ & $+0.084^{*}$ & $+3.273^{*}$ \\
\bottomrule
\end{tabular}%
}
\end{table*}

Qwen CRG has lower adaptive ECE in all 18 comparisons, and the nominal 95\% intervals exclude zero in 16: every verbalized baseline on every benchmark, plus Surrogate LNSP on EnterpriseOps-Gym. Its smaller ECE gains over Surrogate LNSP on SWE-Bench Verified and SkillsBench remain uncertain. Brier and BAS gains are supported in 14 of 18 and 17 of 18 comparisons, respectively. AUROC includes the SWE-Bench Verified losses to two Qwen verbalizers already visible in Table~\ref{tab:main-results}. Thus, task resampling supports the principal calibration and decision-utility gains, but does not establish that every per-benchmark difference is robust. As an exploratory summary, the mean of the six paired differences has a nominal 95\% interval above zero for ECE, Brier, and BAS on all three benchmarks, and for AUROC on EnterpriseOps-Gym and SkillsBench.%

Overall, the paired bootstrap shows that \textbf{CRG's main calibration and decision-utility gains persist under test-task resampling.}%

\subsection{Software Versions and Randomization}
\label{app:reproducibility-details}

SkillsBench and EnterpriseOps-Gym trajectories used OpenHands SDK v1.31.0.
For SWE-Bench Verified, trajectories are source from the published runs on \href{https://index.openhands.dev/}{https://index.openhands.dev/}. Therefore, the precise OpenHands SDK versions vary by model: v1.18.1 with
GPT-5.5, v1.28.0 with Gemini-3.5 Flash, and v1.24.0 with MiniMax-M3.
The Codex and Claude Code transfer runs used the
\texttt{codex-acp} v0.11.1 and \texttt{claude-agent-acp} v0.30.0
adapters, respectively, with OpenHands SDK v1.19.1 and v1.18.0.

The 306 SWE-Bench Verified tasks were selected by difficulty without
random sampling. For EnterpriseOps-Gym, we independently shuffled the
tasks in each of eight domains with seed 42.
Confidence-estimation runs shuffled their fixed datasets
with seed 42; this affected processing order, not task membership.
The paired task bootstrap used 10,000 resamples per comparison and
We did not set an
explicit inference-time seed for agent or estimator LLM calls.

\end{document}